\documentclass{article}

\usepackage[final, nonatbib]{neurips_2020}

\usepackage{pbox}
\usepackage[utf8]{inputenc} 
\usepackage{hyperref}       
\usepackage{url}            
\usepackage{booktabs}       
\usepackage{multirow}
\usepackage{amsfonts}       
\usepackage{amsmath}
\usepackage{graphicx,wrapfig,lipsum}
\usepackage{amsfonts}
\usepackage{booktabs}
\usepackage{empheq}
\usepackage{amsthm}
\usepackage{caption}
\usepackage{subcaption}
\usepackage{mathrsfs}  
\usepackage{amssymb}
\usepackage{thmtools}  
\usepackage{tikz}
\usepackage{nicefrac}
\usepackage{arydshln}
\usetikzlibrary{calc,positioning,matrix, decorations.pathreplacing}

\theoremstyle{plain}
\newtheorem{thm}{Theorem}[section]
\newtheorem*{thm-informal}{Theorem (Informal)}

\newtheorem{lemma}[thm]{Lemma}

\theoremstyle{definition}
\newtheorem{definition}[thm]{Definition}
\newtheorem{remark}[thm]{Remark}

\newcommand{\ts}[2]{\mathcal{T\!S}_{#1}\left(#2\right)}
\newcommand{\norm}[1]{\left\|#1\right\|}
\newcommand{\abs}[1]{\left|#1\right|}
\newcommand{\reals}{\mathbb{R}}
\newcommand{\naturals}{\mathbb{N}}
\newcommand{\bigO}{\mathcal{O}}
\newcommand{\ttfrac}[2]{\nicefrac{#1}{#2}}
\newcommand{\set}[2]{\left\{{#1}\,\left\vert\,{#2}\vphantom{{#1}}\right\}\right.}
\newcommand{\score}[2]{#1\% $\pm$ #2\%}
\newcommand{\bfscore}[2]{\textbf{#1\%} $\mathbf{\pm}$ \textbf{#2\%}}
\newcommand{\aucscore}[2]{#1 $\pm$ #2}
\newcommand{\bfaucscore}[2]{\textbf{#1} $\mathbf{\pm}$ \textbf{#2}}
\newcommand{\boldheading}[1]{

\textbf{#1}\quad}

\title{Neural Controlled Differential Equations for Irregular Time Series}

\author{Patrick Kidger \qquad James Morrill \qquad James Foster \qquad Terry Lyons\\[2pt]
	Mathematical Institute, University of Oxford \\
	The Alan Turing Institute, British Library \\
	\texttt{\{kidger, morrill, foster, tlyons\}@\hspace{0.8pt}maths.ox.ac.uk}}

\begin{document}
	\maketitle
	\begin{abstract}
		Neural ordinary differential equations are an attractive option for modelling temporal dynamics. However, a fundamental issue is that the solution to an ordinary differential equation is determined by its initial condition, and there is no mechanism for adjusting the trajectory based on subsequent observations. Here, we demonstrate how this may be resolved through the well-understood mathematics of \emph{controlled differential equations}. The resulting \emph{neural controlled differential equation} model is directly applicable to the general setting of partially-observed irregularly-sampled multivariate time series, and (unlike previous work on this problem) it may utilise memory-efficient adjoint-based backpropagation even across observations. We demonstrate that our model achieves state-of-the-art performance against similar (ODE or RNN based) models in empirical studies on a range of datasets. Finally we provide theoretical results demonstrating universal approximation, and that our model subsumes alternative ODE models.
	\end{abstract}
	\section{Introduction}
	Recurrent neural networks (RNN) are a popular choice of model for sequential data, such as a time series. The data itself is often assumed to be a sequence of observations from an underlying process, and the RNN may be interpreted as a discrete approximation to some function of this process. Indeed the connection between RNNs and dynamical systems is well-known \cite{FUNAHASHI1993801, rnn-dynamical, E2017, nais-net}.
	
	However this discretisation typically breaks down if the data is irregularly sampled or partially observed, and the issue is often papered over by binning or imputing data \cite{gelman2007dataanalysis}.
	
	A more elegant approach is to appreciate that because the underlying process develops in continuous time, so should our models. For example \cite{Che2018, BRITS, NeuralHawkes, discreteevent} incorporate exponential decay between observations, \cite{gp-adapter1, gp-adapter2} hybridise a Gaussian process with traditional neural network models, \cite{interpolation-prediction} approximate the underlying continuous-time process, and \cite{latent-odes, gru-ode-bayes} adapt recurrent neural networks by allowing some hidden state to evolve as an ODE. It is this last one that is of most interest to us here.
	
	\subsection{Neural ordinary differential equations}
	Neural ordinary differential equations (Neural ODEs) \cite{E2017, neural-odes}, seek to approximate a map $x \mapsto y$ by learning a function $f_\theta$ and linear maps $\ell^1_\theta$, $\ell^2_\theta$ such that
	\begin{equation}\label{eq:node}
	y \approx \ell^1_\theta(z_T),\quad\text{where}\quad z_t = z_0 + \int_0^t f_\theta(z_s) \mathrm{d} s\quad\text{and}\quad z_0 = \ell^2_\theta(x).
	\end{equation}
	Note that $f_\theta$ does not depend explicitly on $s$; if desired this can be included as an extra dimension in $z_s$ \cite[Appendix B.2]{neural-odes}.
	
	Neural ODEs are an elegant concept. They provide an interface between machine learning and the other dominant modelling paradigm that is differential equations. Doing so allows for the well-understood tools of that field to be applied. Neural ODEs also interact beautifully with the manifold hypothesis, as they describe a flow along which to evolve the data manifold.
	
	This description has not yet involved sequential data such as time series. The $t$ dimension in equation \eqref{eq:node} was introduced and then integrated over, and is just an internal detail of the model.
	
	However the presence of this extra (artificial) dimension motivates the question of whether this model can be extended to sequential data such as time series. Given some ordered data $(x_0, \ldots, x_n)$, the goal is to extend the $z_0 = \ell^2_\theta(x)$ condition of equation \eqref{eq:node} to a condition resembling ``$z_0 = \ell(x_0), \ldots, z_n = \ell(x_n)$'', to align the introduced $t$ dimension with the natural ordering of the data.
	
	The key difficulty is that equation \eqref{eq:node} defines an ordinary differential equation; once $\theta$ has been learnt, then the solution of equation \eqref{eq:node} is determined by the initial condition at $z_0$, and there is no direct mechanism for incorporating data that arrives later \cite{nais-net}.
	
	
	However, it turns out that the resolution of this issue -- how to incorporate incoming information -- is already a well-studied problem in mathematics, in the field of rough analysis, which is concerned with the study of \emph{controlled differential equations}.\footnote{Not to be confused with the similarly-named but separate field of control theory.} See for example \cite{lyons1998differential, lyons2014rough, hambly2010uniqueness, chevyrev2018signature}. An excellent introduction is \cite{levy-lyons}. A comprehensive textbook is \cite{FritzVictoir10}.
	
	We will not assume familiarity with either controlled differential equations or rough analysis. The only concept we will rely on that may be unfamiliar is that of a Riemann--Stieltjes integral.
	
	\subsection{Contributions}
	We demonstrate how controlled differential equations may extend the Neural ODE model, which we refer to as the \emph{neural controlled differential equation} (Neural CDE) model. Just as Neural ODEs are the continuous analogue of a ResNet, the Neural CDE is the continuous analogue of an RNN.
	
	The Neural CDE model has three key features. One, it is capable of processing incoming data, which may be both irregularly sampled and partially observed. Two (and unlike previous work on this problem) the model may be trained with memory-efficient adjoint-based backpropagation even across observations. Three, it demonstrates state-of-the-art performance against similar (ODE or RNN based) models, which we show in empirical studies on the CharacterTrajectories, PhysioNet sepsis prediction, and Speech Commands datasets.

	We provide additional theoretical results showing that our model is a universal approximator, and that it subsumes apparently-similar ODE models in which the vector field depends directly upon continuous data.
	
	Our code is available at \url{https://github.com/patrick-kidger/NeuralCDE}. We have also released a library \texttt{torchcde}, at \url{https://github.com/patrick-kidger/torchcde}
	
	\section{Background}
	Let $\tau, T \in \reals$ with $\tau < T$, and let $v, w \in \naturals$. Let $X \colon [\tau, T] \to \reals^v$ be a continuous function of bounded variation; for example this is implied by $X$ being Lipschitz. Let $\zeta \in \reals^w$. Let $f \colon \reals^w \to \reals^{w \times v}$ be continuous.
	
	Then we may define a continuous path $z \colon [\tau, T] \to \reals^w$ by $z_\tau = \zeta$ and
	\begin{equation}\label{eq:cde}
	z_t = z_\tau + \int_\tau^t f(z_s) \mathrm{d} X_s\quad\text{for $t \in (\tau, T]$},
	\end{equation}
	where the integral is a Riemann--Stieltjes integral. As $f(z_s) \in \reals^{w \times v}$ and $X_s \in \reals^v$, the notation ``$f(z_s) \mathrm{d}X_s$'' refers to matrix-vector multiplication. The subscript notation refers to function evaluation, for example as is common in stochastic calculus.
	
	Equation \eqref{eq:cde} exhibits global existence and uniqueness subject to global Lipschitz conditions on $f$; see \cite[Theorem 1.3]{levy-lyons}. We say that equation \eqref{eq:cde} is a controlled differential equation (CDE) which is controlled or driven by $X$.
	
	
	\section{Method}	
	Suppose for simplicity that we have a fully-observed but potentially irregularly sampled time series $\mathbf{x} = ((t_0, x_0), (t_1, x_1), \ldots, (t_n, x_n))$, with each $t_i \in \reals$ the timestamp of the observation $x_i \in \reals^v$, and $t_0 < \cdots < t_n$. (We will consider partially-observed data later.)
	
	Let $X \colon [t_0, t_n] \to \reals^{v + 1}$ be the natural cubic spline with knots at $t_0, \ldots, t_n$ such that $X_{t_i} = (x_i, t_i)$. As $\mathbf{x}$ is often assumed to be a discretisation of an underlying process, observed only through $\mathbf{x}$, then $X$ is an approximation to this underlying process. Natural cubic splines have essentially the minimum regularity for handling certain edge cases; see Appendix \ref{appendix:other-schemes-for-x} for the technical details.
		
	Let $f_\theta \colon \reals^w \to \reals^{w \times (v + 1)}$ be any neural network model depending on parameters $\theta$. The value $w$ is a hyperparameter describing the size of the hidden state. Let $\zeta_\theta \colon \reals^{v + 1} \to \reals^w$ be any neural network model depending on parameters $\theta$.
	
	Then we define the \emph{neural controlled differential equation} model as the solution of the CDE
	\begin{equation}\label{eq:ncde}
	z_t = z_{t_0} + \int_{t_0}^t f_\theta(z_s) \mathrm{d} X_s\quad\text{for $t \in (t_0, t_n]$},
	\end{equation}
	where $z_{t_0} = \zeta_\theta(x_0, t_0)$. This initial condition is used to avoid translational invariance. Analogous to RNNs, the output of the model may either be taken to be the evolving process $z$, or the terminal value $z_{t_n}$, and the final prediction should typically be given by a linear map applied to this output.
	
	The resemblance between equations \eqref{eq:node} and \eqref{eq:ncde} is clear. The essential difference is that equation \eqref{eq:ncde} is driven by the data process $X$, whilst equation \eqref{eq:node} is driven only by the identity function $\iota \colon \reals \to \reals$. In this way, the Neural CDE is naturally adapting to incoming data, as changes in $X$ change the local dynamics of the system. See Figure \ref{figure:picture}.
	
\newcommand{\mainpicturedata}{
	\draw[thick, ->] (-1, 0) -- (5, 0);
	\node at (0, 0)[circle,fill, inner sep=1.5pt] {};
	\node at (0.6, 0)[circle,fill, inner sep=1.5pt] {};
	\node at (2, 0)[circle,fill, inner sep=1.5pt] {};
	\node at (4, 0)[circle,fill, inner sep=1.5pt] {};
	\node at (0, -0.3) {$t_1$};
	\node at (0.6, -0.3) {$t_2$};
	\node at (2, -0.3) {$t_3$};
	\node at (3, -0.3) {$\cdots$};
	\node at (4, -0.3) {$t_n$};
	\node at (5.6, 0) {Time};
	
	\draw[blue!80!white, thick, cap=round] (-1, 0.3) .. controls ++(45:0.5) and ++(185:0.2) .. (0, 1);
	\draw[blue!80!white, thick, cap=round] (0, 1) .. controls ++(185+180:0.2) and ++(135:0.4) .. (0.6, 0.7);
	\draw[blue!80!white, thick, cap=round] (0.6, 0.7) .. controls ++(135+180:0.4) and ++ (190:0.3) .. (2, 0.6);
	\draw[blue!80!white, thick, cap=round] (2, 0.6) .. controls ++(190+180:0.3) and ++ (135:0.7) .. (4, 0.8);
	\draw[blue!80!white, thick, cap=round] (4, 0.8) .. controls ++(135+180:0.7) and ++(200:0.2) .. (5, 0.6);
	\node at (0, 1)[circle,draw=black, fill=cyan, inner sep=1.5pt] {};
	\node at (0.6, 0.7)[circle,draw=black, fill=cyan, inner sep=1.5pt] {};
	\node at (2, 0.6)[circle,draw=black, fill=cyan, inner sep=1.5pt] {};
	\node at (4, 0.8)[circle,draw=black, fill=cyan, inner sep=1.5pt] {};
	\node at (0, 0.7) {$x_1$};
	\node at (0.6, 0.4) {$x_2$};
	\node at (2, 0.3) {$x_3$};
	\node at (4, 0.5) {$x_n$};
	\node at (5.6, 0.6) {Data $\mathbf{x}$};
	}
	
	\begin{figure}[t]
	\begin{subfigure}[b]{0.49\textwidth}\centering
	\begin{tikzpicture}[scale=0.9]
	\mainpicturedata
	
	\draw[olive, thick, cap=round, ->] (0, 2) -- (0, 1.8);
	\draw[green!80!black, thick, cap=round] (0, 1.8) .. controls (0.3, 1.4) and (0.7, 2.4) .. (0.6,2.2);
	\draw[olive, thick, cap=round, ->] (0.6,2.2) -- (0.6,1.7);
	\draw[green!80!black, thick, cap=round] (0.6,1.7) .. controls ++(10:0.7) and ++(160:0.6) .. (2,1.9);
	\draw[olive, thick, cap=round, ->] (2,1.9) -- (2,2.5);
	\draw[green!80!black, thick, cap=round] (2,2.5) .. controls ++(0:0.6) and ++(210:0.9) .. (4,2);
	\draw[olive, thick, cap=round, ->] (4,2) -- (4, 2.2);
	
	\draw[dashed, ->] (0, 1.1) -- (0, 1.7);
	\draw[dashed, ->] (0.6, 0.8) -- (0.6, 1.65);
	\draw[dashed, ->] (2, 0.7) -- (2, 1.85);
	\draw[dashed, ->] (4, 0.9) -- (4, 1.95);
	
	\node at (5.6, 2.2) {Hidden state $z$};
	\end{tikzpicture}
	\end{subfigure}
	\hspace{-0.2em}
	\begin{subfigure}[b]{0.49\textwidth}\centering
	\begin{tikzpicture}[scale=0.9]
	\mainpicturedata
	
	\draw[orange!95!black, thick, cap=round] (0, 2) .. controls ++(0:0.2) and ++(150:0.4) .. (1, 1.7);
	\draw[orange!95!black, thick, cap=round] (1, 1.7) .. controls ++(150+180:0.4) and ++ (185:0.3) .. (2, 1.6);
	\draw[orange!95!black, thick, cap=round] (2, 1.6) .. controls ++(185+180:0.3) and ++ (165:0.7) .. (4, 1.8);
	\node at (5.6, 1.8) {Path $X$};
	
	\draw[dashed, ->] (0, 1.1) -- (0, 1.95);
	\draw[dashed, ->] (0.6, 0.8) -- (0.6, 1.8);
	\draw[dashed, ->] (2, 0.7) -- (2, 1.55);
	\draw[dashed, ->] (4, 0.9) -- (4, 1.75);
	
	\draw[dashed, ->] (0, 1.1) -- (0.4, 1.85);
	\draw[dashed, ->] (0.6, 0.8) -- (0.2, 1.9);
	\draw[dashed, ->] (0.6, 0.8) -- (1, 1.65);
	\draw[dashed, ->] (0.6, 0.8) -- (1.4, 1.5);
	\draw[dashed, ->] (0.6, 0.8) -- (1.8, 1.5);
	\draw[dashed, ->] (2, 0.7) -- (1.2, 1.55);
	\draw[dashed, ->] (2, 0.7) -- (1.6, 1.5);
	\draw[dashed, ->] (2, 0.7) -- (2.4, 1.6);
	\draw[dashed, ->] (4, 0.9) -- (3.6, 1.8);
	
	\draw[green!80!black, thick, cap=round] (0, 3) .. controls ++(-30:0.4) and ++(210:0.5) .. (1, 3);
	\draw[green!80!black, thick, cap=round] (1, 3) .. controls ++(210+180:0.5) and ++(180:0.7) .. (2.5, 3.5);
	\draw[green!80!black, thick, cap=round] (2.5, 3.5) .. controls ++(0:0.7) and ++(150:0.5) .. (4, 3);
	
	\draw[dashed, ->] (0, 2.05) -- (0, 2.95);
	\draw[dashed, ->] (0.2,2) -- (0.2,2.85);
	\draw[dashed, ->] (0.4,1.95) -- (0.4,2.8);
	\draw[dashed, ->] (0.6,1.9) -- (0.6,2.8);
	\draw[dashed, ->] (0.8,1.85) -- (0.8,2.85);
	\draw[dashed, ->] (1,1.75) -- (1,2.95);
	\draw[dashed, ->] (1.2, 1.7) -- (1.2, 3.05);
	\draw[dashed, ->] (1.4, 1.6) -- (1.4, 3.15);
	\draw[dashed, ->] (1.6, 1.6) -- (1.6, 3.25);
	\draw[dashed, ->] (1.8, 1.6) -- (1.8, 3.35);
	\draw[dashed, ->] (2, 1.65) -- (2, 3.4);
	\draw[dashed, ->] (2.2, 1.7) -- (2.2, 3.45);
	\draw[dashed, ->] (2.4, 1.75) -- (2.4, 3.45);
	\draw[dashed, ->] (3.6, 1.9) -- (3.6, 3.2);
	\draw[dashed, ->] (3.8, 1.9) -- (3.8, 3.05);
	\draw[dashed, ->] (4, 1.85) -- (4, 2.95);
	\node at (5.6,3) {Hidden state $z$};
	\end{tikzpicture}
	\end{subfigure}
	\caption{Some data process is observed at times $t_1, \ldots, t_n$ to give observations $x_1, \ldots, x_n$. It is otherwise unobserved. \textbf{Left:} Previous work has typically modified hidden state at each observation, and perhaps continuously evolved the hidden state between observations. \textbf{Right: } In contrast, the hidden state of the Neural CDE model has continuous dependence on the observed data.}\label{figure:picture}
	\end{figure}
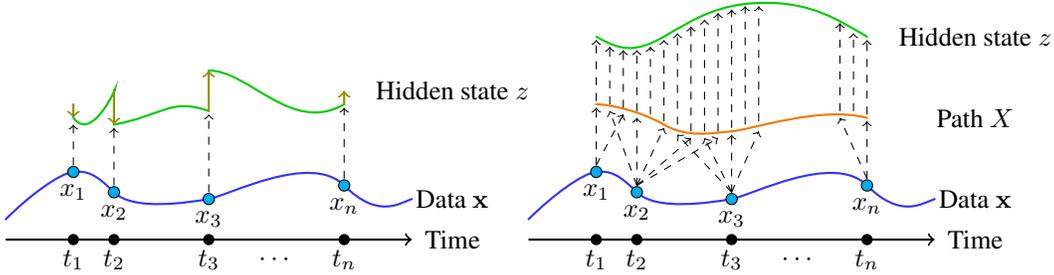
	
	\subsection{Universal Approximation}
	It is a famous theorem in CDEs that in some sense they represent general functions on streams \cite[Theorem 4.2]{perezarribas2018}, \cite[Proposition A.6]{kidger2019deep}. This may be applied to show that Neural CDEs are universal approximators, which we summarise in the following informal statement.
	
	\begin{thm-informal}
	The action of a linear map on the terminal value of a Neural CDE is a universal approximator from $\{\text{\emph{sequences in }} \reals^v\}$ to $\reals$.
	\end{thm-informal}
	
	Theorem \ref{thm:datawise-universal-approximation} in Appendix \ref{appendix:universal} gives a formal statement and a proof, which is somewhat technical. The essential idea is that CDEs may be used to approximate bases of functions on path space.
	
	\subsection{Evaluating the Neural CDE model}\label{subsection:evaluate}
	Evaluating the Neural CDE model is straightforward. In our formulation above, $X$ is in fact not just of bounded variation but is differentiable. In this case, we may define
	\begin{equation}\label{eq:g-theta-x}
	g_{\theta, X}(z, s) = f_\theta(z) \frac{\mathrm{d}X}{\mathrm{d}s}(s),
	\end{equation}
	so that for $t \in (t_0, t_n]$,
	\begin{equation}\label{eq:ncde-to-node}
	z_t = z_{t_0} + \int_{t_0}^t f_\theta(z_s) \mathrm{d} X_s = z_{t_0} + \int_{t_0}^t f_\theta(z_s) \frac{\mathrm{d}X}{\mathrm{d}s}(s) \mathrm{d}s = z_{t_0} + \int_{t_0}^t g_{\theta, X}(z_s, s) \mathrm{d}s.
	\end{equation}
	
	Thus it is possible to solve the Neural CDE using the same techniques as for Neural ODEs. 
	In our experiments, we were able to straightforwardly use the already-existing \texttt{torchdiffeq} package \cite{torchdiffeq} without modification.
	
	\subsection{Comparison to alternative ODE models}\label{section:contrnn}
	For the reader not familiar with CDEs, it might instead seem more natural to replace $g_{\theta, X}$ with some $h_\theta(z, X_s)$ that is directly applied to and potentially nonlinear in $X_s$. Indeed, such approaches have been suggested before, in particular to derive a ``GRU-ODE'' analogous to a GRU \cite{gru-ode-bayes, gruode}.
	
	However, it turns out that something is lost by doing so, which we summarise in the following statement.
	
	\begin{thm-informal}
	Any equation of the form $z_t = z_0 + \int_{t_0}^t h_\theta(z_s, X_s) \mathrm{d}s$ may be represented exactly by a Neural CDE of the form $z_t = z_0 + \int_{t_0}^t f_\theta(z_s) \mathrm{d}X_s$. However the converse statement is not true.
	\end{thm-informal}
	
	Theorem \ref{cont-rnn-thm} in Appendix \ref{appendix:nonlinear-g} provides the formal statement and proof. The essential idea is that a Neural CDE can easily represent the identity function between paths, whilst the alternative cannot.
	
	In our experiments, we find that the Neural CDE substantially outperforms the GRU-ODE, which we speculate is a consequence of this result.
	
	\subsection{Training via the adjoint method}\label{subsection:train}
	An attractive part of Neural ODEs is the ability to train via adjoint backpropagation, see \cite{neural-odes, pontryagin, Giles2000, Hager2000}, which uses only $\bigO(H)$ memory in the time horizon $L = t_n - t_0$ and the memory footprint $H$ of the vector field. This is contrast to directly backpropagating through the operations of an ODE solver, which requires $\bigO(LH)$ memory.
	
	Previous work on Neural ODEs for time series, for example \cite{latent-odes}, has interrupted the ODE to make updates at each observation. Adjoint-based backpropagation cannot be performed across the jump, so this once again requires $\bigO(LH)$ memory. 
	
	In contrast, the $g_{\theta, X}$ defined by equation \eqref{eq:g-theta-x} continuously incorporates incoming data, without interrupting the differential equation, and so adjoint backpropagation may be performed. This requires only $\bigO(H)$ memory. The underlying data unavoidably uses an additional $\bigO(L)$ memory. Thus training the Neural CDE has an overall memory footprint of just $\bigO(L + H)$.	
	
	We do remark that the adjoint method should be used with care, as some systems are not stable to evaluate in both the forward and backward directions \cite{anode, anode2}. The problem of finite-time blow-up is at least not a concern, given global Lipschitz conditions on the vector field \cite[Theorem 1.3]{levy-lyons}. Such a condition will be satisfied if $f_\theta$ uses ReLU or tanh nonlinearities, for example.
	
	\subsection{Intensity as a channel}\label{section:observational-intensity}
	It has been observed that the frequency of observations may carry information \cite{Che2018}. For example, doctors may take more frequent measurements of patients they believe to be at greater risk. Some previous work has for example incorporated this information by learning an intensity function \cite{interpolation-prediction, latent-odes, neural-odes}.
	
	We instead present a simple \emph{non-learnt} procedure, that is compatible with Neural CDEs. Simply concatenate the index $i$ of $x_i$ together with $x_i$, and then construct a path $X$ from the pair $(i, x_i)$, as before. The channel of $X$ corresponding to these indices then corresponds to the \emph{cumulative} intensity of observations.
	
	As the derivative of $X$ is what is then used when evaluating the Neural CDE model, as in equation \eqref{eq:ncde-to-node}, then it is the intensity itself that then determines the vector field.
	
	\subsection{Partially observed data}\label{section:partial}
	One advantage of our formulation is that it naturally adapts to the case of partially observed data. Each channel may independently be interpolated between observations to define $X$ in exactly the same manner as before.
	
	In this case, the procedure for measuring observational intensity in Section \ref{section:observational-intensity} may be adjusted by instead having a separate observational intensity channel $c_i$ for each original channel $o_i$, such that $c_i$ increments every time an observation is made in $o_i$.
	
	\subsection{Batching}
	Given a batch of training samples with observation times drawn from the same interval $[t_0, t_n]$, we may interpolate each $\mathbf{x}$ to produce a continuous $X$, as already described. Each path $X$ is what may then be batched together, regardless of whether the underlying data is irregularly sampled or partially observed. Batching is thus efficient for the Neural CDE model.
	
	\section{Experiments}	
	We benchmark the Neural CDE against a variety of existing models.
	
	These are: GRU-$\Delta$t, which is a GRU with the time difference between observations additionally used as an input; GRU-D \cite{Che2018}, which modifies the GRU-$\Delta$t with learnt exponential decays between observations; GRU-ODE \cite{gru-ode-bayes, gruode}, which is an ODE analogous to the operation of a GRU and uses $X$ as its input; ODE-RNN \cite{latent-odes}, which is a GRU-$\Delta$t model which additionally applies a learnt Neural ODE to the hidden state between observations. Every model then used a learnt linear map from the final hidden state to the output, and was trained with cross entropy or binary cross entropy loss.
	
	The GRU-$\Delta$t represents a straightforward baseline, the GRU-ODE is an alternative ODE model that is thematically similar to a Neural CDE, and the GRU-D and ODE-RNNs are state-of-the-art models for these types of problems. To avoid unreasonably extensive comparisons we have chosen to focus on demonstrating superiority within the class of ODE and RNN based models to which the Neural CDE belongs. These models were selected to collectively be representative of this class.
	
	
	Each model is run five times, and we report the mean and standard deviation of the test metrics.
	
	For every problem, the hyperparameters were chosen by performing a grid search to optimise the performance of the baseline ODE-RNN model. Equivalent hyperparameters were then used for every other model, adjusted slightly so that every model has a comparable number of parameters.
	
	Precise experimental details may be found in Appendix \ref{appendix:experimental-details}, regarding normalisation, architectures, activation functions, optimisation, hyperparameters, regularisation, and so on.
	
	\subsection{Varying amounts of missing data on CharacterTrajectories}
	We begin by demonstrating the efficacy of Neural CDEs on irregularly sampled time series.
	
	To do this, we consider the CharacterTrajectories dataset from the UEA time series classification archive \cite{uea}. This is a dataset of 2858 time series, each of length 182, consisting of the $x, y$ position and pen tip force whilst writing a Latin alphabet character in a single stroke. The goal is to classify which of 20 different characters are written.
	
	We run three experiments, in which we drop either 30\%, 50\% or 70\% of the data. The observations to drop are selected uniformly at random and independently for each time series. Observations are removed across channels, so that the resulting dataset is irregularly sampled but completely observed. The randomly removed data is the same for every model and every repeat.
	
	The results are shown in Table \ref{table:charactertrajectories}. The Neural CDE outperforms every other model considered, and furthermore it does so whilst using an order of magnitude less memory. The GRU-ODE does consistently poorly despite being the most theoretically similar model to a Neural CDE. Furthermore we see that even as the fraction of dropped data increases, the performance of the Neural CDE remains roughly constant, whilst the other models all start to decrease.
	
	Further experimental details may be found in Appendix \ref{appendix:uea}.
	
	\begin{table}\centering
	\caption{Test accuracy (mean $\pm$ std, computed across five runs) and memory usage on CharacterTrajectories. Memory usage is independent of repeats and of amount of data dropped.}\label{table:charactertrajectories}
	\begin{tabular}{@{}lcccc@{}}
	\toprule
	\multirow{2}{*}{Model} & \multicolumn{3}{c}{Test Accuracy} & \multirow{2}{*}{\pbox{20cm}{Memory \\ usage (MB)}} \\ \cmidrule(l){2-4} 
                                            & 30\% dropped & 50\% dropped & 70\% dropped \\ \midrule
	GRU-ODE & \score{92.6}{1.6} & \score{86.7}{3.9} & \score{89.9}{3.7} & 1.5 \\
	GRU-$\Delta$t & \score{93.6}{2.0} & \score{91.3}{2.1} & \score{90.4}{0.8} & 15.8\\
	GRU-D & \score{94.2}{2.1} & \score{90.2}{4.8} & \score{91.9}{1.7} & 17.0 \\
	ODE-RNN & \score{95.4}{0.6} & \score{96.0}{0.3} & \score{95.3}{0.6} & 14.8 \\[0.1ex] \hdashline\noalign{\vskip 0.7ex}
        Neural CDE (ours) & \bfscore{98.7}{0.8} & \bfscore{98.8}{0.2} & \bfscore{98.6}{0.4} & \textbf{1.3} \\ \bottomrule
	\end{tabular}
	\end{table}
	
	\subsection{Observational intensity with PhysioNet sepsis prediction}
	Next we consider a dataset that is both irregularly sampled and partially observed, and investigate the benefits of observational intensity as discussed in Sections \ref{section:observational-intensity} and \ref{section:partial}.
	
	We use data from the PhysioNet 2019 challenge on sepsis prediction \cite{sepsis1, sepsis2}. This is a dataset of 40335 time series of variable length, describing the stay of patients within an ICU. Measurements are made of 5 static features such as age, and 34 time-dependent features such as respiration rate or creatinine concentration in the blood, down to an hourly resolution. Most values are missing; only $10.3\%$ of values are observed. We consider the first 72 hours of a patient's stay, and consider the binary classification problem of predicting whether they develop sepsis over the course of their entire stay (which is as long as a month for some patients).
	
	We run two experiments, one with observational intensity, and one without. For the Neural CDE and GRU-ODE models, observational intensity is continuous and on a per-channel basis as described in Section \ref{section:partial}. For the ODE-RNN, GRU-D, and GRU-$\Delta$t models, observational intensity is given by appending an observed/not-observed mask to the input at each observation.\footnote{As our proposed observational intensity goes via a derivative, these each contain the same information.}\footnote{Note that the ODE-RNN, GRU-D and GRU-$\Delta$t models always receive the time difference between observations, $\Delta$t, as an input. Thus even in the no observational intensity case, they remain aware of the irregular sampling of the data, and so this case not completely fair to the Neural CDE and GRU-ODE models.}  The initial hidden state of every model is taken to be a function (a small single hidden layer neural network) of the static features.
	
	The results are shown in Table \ref{table:sepsis}. As the dataset is highly imbalanced (5\% positive rate), we report AUC rather than accuracy. When observational intensity is used, then the Neural CDE produces the best AUC overall, although the ODE-RNN and GRU-$\Delta$t models both perform well. The GRU-ODE continues to perform poorly.
	
	Without observational intensity then every model performs substantially worse, and in particular we see that the benefit of including observational intensity is particularly dramatic for the Neural CDE.
	
	As before, the Neural CDE remains the most memory-efficient model considered.
	
	\begin{table}
	\centering
	\caption{Test AUC (mean $\pm$ std, computed across five runs) and memory usage on PhysioNet sepsis prediction. `OI' refers to the inclusion of observational intensity, `No OI' means without it. Memory usage is independent of repeats.}\label{table:sepsis}
	\begin{tabular}{@{}lcccccc@{}}
	\toprule
	\multirow{2}{*}{Model} & \multicolumn{2}{c}{Test AUC} & \multicolumn{2}{c}{Memory usage (MB)}\\ \cmidrule(l){2-5} 
                                            & OI & No OI & OI & No OI \\ \midrule
	GRU-ODE & \aucscore{0.852}{0.010} & \aucscore{0.771}{0.024} & 454 & 273 \\
	GRU-$\Delta$t & \aucscore{0.878}{0.006} & \aucscore{0.840}{0.007} & 837 & 826 \\
	GRU-D & \aucscore{0.871}{0.022} & \bfaucscore{0.850}{0.013} & 889 & 878 \\
	ODE-RNN & \aucscore{0.874}{0.016} & \aucscore{0.833}{0.020} & 696 & 686 \\[0.1ex] \hdashline\noalign{\vskip 0.7ex}
        Neural CDE (ours) & \bfaucscore{0.880}{0.006} & \aucscore{0.776}{0.009} & \textbf{244} & \textbf{122} \\ \bottomrule
	\end{tabular}
	\end{table}
	
	Further experimental details can be found in Appendix \ref{appendix:sepsis}.
	
	\subsection{Regular time series with Speech Commands}
	Finally we demonstrate the efficacy of Neural CDE models on regularly spaced, fully observed time series, where we might hypothesise that the baseline models will do better.
	
	We used the Speech Commands dataset \cite{speechcommands}. This consists of one-second audio recordings of both background noise and spoken words such as `left', `right', and so on. We used 34975 time series corresponding to ten spoken words so as to produce a balanced classification problem. 
	We preprocess the dataset by computing mel-frequency cepstrum coefficients so that each time series is then regularly spaced with length 161 and 20 channels.
	
	\begin{wraptable}{l}{75mm}\centering
	\vspace{-0.4em}
	\caption{Test Accuracy (mean $\pm$ std, computed across five runs) and memory usage on Speech Commands. Memory usage is independent of repeats.}\label{table:speechcommands}
	\begin{tabular}{@{}lcc@{}}
	\toprule
	Model & Test Accuracy & \pbox{20cm}{Memory \\ usage (GB)}\\ \midrule
	GRU-ODE & \score{47.9}{2.9} & \textbf{0.164} \\
	GRU-$\Delta$t & \score{43.3}{33.9} & 1.54 \\
	GRU-D & \score{32.4}{34.8} & 1.64 \\
	ODE-RNN & \score{65.9}{35.6} & 1.40\\[0.1ex] \hdashline\noalign{\vskip 0.7ex}
        Neural CDE (ours) & \bfscore{89.8}{2.5} & 0.167 \\ \bottomrule
	\end{tabular}
	\vspace{-0.5em}
	\end{wraptable}
		
	The results are shown in Table \ref{table:speechcommands}. We observed that the Neural CDE had the highest performance, whilst using very little memory. The GRU-ODE consistently failed to perform. The other benchmark models surprised us by exhibiting a large variance on this problem, due to sometimes failing to train, and we were unable to resolve this by tweaking the optimiser. The best GRU-$\Delta$t, \mbox{GRU-D} and ODE-RNN models did match the performance of the Neural CDE, suggesting that on a regularly spaced problem all approaches can be made to work equally well.
	
	In contrast, the Neural CDE model produced consistently good results every time. Anecdotally this aligns with what we observed over the course of all of our experiments, which is that the Neural CDE model usually trained quickly, and was robust to choice of optimisation hyperparameters. We stress that we did not perform a formal investigation of this phenomenen.
	
%
	
	
	Further experimental details can be found in Appendix \ref{appendix:speechcommands}.
	
	\section{Related work}
	In \cite{latent-odes, gru-ode-bayes} the authors consider interrupting a Neural ODE with updates from a recurrent cell at each observation, and were in fact the inspiration for this paper. Earlier work \cite{Che2018, BRITS, NeuralHawkes, discreteevent} use intra-observation exponential decays, which are a special case. \cite{njsde} consider something similar by interrupting a Neural ODE with stochastic events.
	
	SDEs and CDEs are closely related, and several authors have introduced Neural SDEs. \cite{josef, nsde-generative, stochasticDNF} treat them as generative models for time series and seek to model the data distribution. \cite{nsde-normalisation, stochasticnode} investigate using stochasticity as a regularizer, and demonstrate better performance by doing so. \cite{stochasticvectorfield} use random vector fields so as to promote simpler trajectories, but do not use the `SDE' terminology.
	
	Adjoint backpropagation needs some work to apply to SDEs, and so \cite{scalable-sde, nsde-basic, roughstochasticNF} all propose methods for training Neural SDEs. We would particularly like to highlight the elegant approach of \cite{roughstochasticNF}, who use the pathwise treatment given by rough analysis to approximate Brownian noise, and thus produce a random Neural ODE which may be trained in the usual way; such approaches may also avoid the poor convergence rates of SDE solvers as compared to ODE solvers.

	Other elements of the theory of rough analysis and CDEs have also found machine learning applications. Amongst others, \cite{kidger2019deep, primer2016, PerezArribas2018different, fermanian2019embedding, morrill2019sepsis, jeremythesis, pathsig, generalisedsignature} discuss applications of the signature transform to time series problems, and \cite{logsig-rnn} investigate the related logsignature transform. \cite{kiraly2019kernels} develop a kernel for time series using this methodology, and \cite{toth2019gp} apply this kernel to Gaussian processes. \cite{signatory} develop software for these approaches tailored for machine learning.
	
	There has been a range of work seeking to improve Neural ODEs. \cite{Quaglino2020SNODE, trainingnode} investigate speed-ups to the training proecedure, \cite{stable-neural-flows} develop an energy-based Neural ODE framework, and \cite{anode} demonstrate potential pitfalls with adjoint backpropagation. \cite{anode2, dissecting} consider ways to vary the network parameters over time. \cite{trainingnode, robust-node} consider how a Neural ODE model may be regularised (see also the stochastic regularisation discussed above). This provides a wide variety of techniques, and we are hopeful that some of them may additionally carry over to the Neural CDE case.
	
	\section{Discussion}
	\subsection{Considerations}
	There are two key elements of the Neural CDE construction which are subtle, but important.

	\boldheading{Time as a channel}
	CDEs exhibit a \emph{tree-like invariance} property \cite{hambly2010uniqueness}. What this means, roughly, is that a CDE is blind to speed at which $X$ is traversed. Thus merely setting $X_{t_i} = x_i$ would not be enough, as time information is only incorporated via the parameterisation. This is why time is explicitly included as a channel via $X_{t_i} = (x_i, t_i)$.
	
	\boldheading{Initial value networks}
	The initial hidden state $z_{t_0}$ should depend on $X_{t_0} = (x_0, t_0)$. Otherwise, the Neural CDE will depend upon $X$ only through its derivative $\ttfrac{\mathrm{d}X}{\mathrm{d}t}$, and so will be translationally invariant. An alternative would be to append another channel whose first derivative includes translation-sensitive information, for example by setting $X_{t_i} = (x_i, t_i, t_i x_0)$.
	
	\subsection{Performance tricks}
	We make certain (somewhat anecdotal) observations of tricks that seemed to help performance.
	
	\boldheading{Final tanh nonlinearity}
	We found it beneficial to use a tanh as a final nonlinearity for the vector field $f_\theta$ of a Neural CDE model. Doing so helps prevent extremely large initial losses, as the tanh constrains the rate of change of the hidden state. This is analogous to RNNs, where the key feature of GRUs and LSTMs are procedures to constrain the rate of change of the hidden state.
	
	\boldheading{Layer-wise learning rates}
	We found it beneficial to use a larger ($\times$10--100) learning rate for the linear layer on top of the output of the Neural CDE, than for the vector field $f_\theta$ of the Neural CDE itself. This was inspired by the observation that the final linear layer has (in isolation) only a convex optimisation problem to solve.\footnote{In our experiments we applied this learning rate to the linear layer on top of every model, not the just the Neural CDE, to ensure a fair comparison.}
	
	
	\subsection{Limitations} \label{section:limitations}
	\boldheading{Speed of computation}
	We found that Neural CDEs were typically slightly faster to compute than the ODE-RNN model of \cite{latent-odes}. (This is likely to be because in an Neural CDE, steps of the numerical solver can be made across observations, whilst the ODE-RNN must interrupt its solve at each observation.)
	
	However, Neural CDEs were still roughly fives times slower than RNN models. We believe this is largely an implementation issue, as the implementation via \texttt{torchdiffeq} is in Python, and by default uses double-precision arithmetic with variable step size solvers, which we suspect is unnecessary for most practical tasks.
	
	\boldheading{Number of parameters}
	If the vector field $f_\theta \colon \reals^w \to \reals^{w \times (v + 1)}$ is a feedforward neural network, with final hidden layer of size $\omega$, then the number of scalars for the final affine transformation is of size $\bigO(\omega v w)$, which can easily be very large. In our experiments we have to choose small values of $w$ and $\omega$ for the Neural CDE to ensure that the number of parameters is the same across models.
	
	We did experiment with representing the final linear layer as an outer product of transformations $\reals^w \to \reals^w$ and $\reals^w \to \reals^{v + 1}$. This implies that the resulting matrix is rank-one, and reduces the number of parameters to just $\bigO(\omega (v + w))$, but unfortunately we found that this hindered the classification performance of the model.
	
	\subsection{Future work}	
	\boldheading{Vector field design}
	The vector fields $f_\theta$ that we consider are feedforward networks. More sophisticated choices may allow for improved performance, in particular to overcome the trilinearity issue just discussed.
	
	\boldheading{Modelling uncertainty}
	As presented here, Neural CDEs do not give any measure of uncertainty about their predictions. Such extensions are likely to be possible, given the close links between CDEs and SDEs, and existing work on Neural SDEs \cite{josef, nsde-generative, stochasticDNF, nsde-normalisation, stochasticnode, scalable-sde, nsde-basic, roughstochasticNF, imanol2020sde}.
	
	\boldheading{Numerical schemes}
	In this paper we integrated the Neural CDE by reducing it to an ODE. The field of numerical CDEs is relatively small -- to the best of our knowledge \cite{lyons2014rough, logode1, logode2, gyurkothesis, arendthesis, dimensionfree, yang2015roughtaylor, jamesthesis, foster2020odemethod} constitute essentially the entire field, and are largely restricted to rough controls. Other numerical methods may be able to exploit the CDE structure to improve performance.
	
	\boldheading{Choice of $X$} Natural cubic splines were used to construct the path $X$ from the time series $\mathbf{x}$. However, these are not causal. That is, $X_t$ depends upon the value of $x_i$ for $t < t_i$. This makes it infeasible to apply Neural CDEs in real-time settings, for which $X_t$ is needed before $x_i$ has been observed. Resolving this particular issue is a topic on which we have follow-up work planned.
	
	\boldheading{Other problem types}
	Our experiments here involved only classification problems. There was no real reason for this choice, and we expect Neural CDEs to be applicable more broadly.
	
	\subsection{Related theories}	
	\boldheading{Rough path theory}
	The field of rough path theory, which deals with the study of CDEs, is much larger than the small slice that we have used here. It is likely that further applications may serve to improve Neural CDEs. A particular focus of rough path theory is how to treat functions that must be sensitive to the order of events in a particular (continuous) way.
	
	\boldheading{Control theory}
	Despite their similar names, and consideration of similar-looking problems, control theory and controlled differential equations are essentially separate fields. Control theory has clear links and applications that may prove beneficial to models of this type.
	
	\boldheading{RNN theory}
	Neural CDEs may be interpreted as continuous-time versions of RNNs. CDEs thus offer a theoretical construction through which RNNs may perhaps be better understood. Conversely, what is known about RNNs may have applications to improve Neural CDEs.
	
	\section{Conclusion}
	We have introduced a new class of continuous-time time series models, Neural CDEs. Just as Neural ODEs are the continuous analogue of ResNets, the Neural CDE is the continuous time analogue of an RNN. The model has three key advantages: it operates directly on irregularly sampled and partially observed multivariate time series, it demonstrates state-of-the-art performance, and it benefits from memory-efficient adjoint-based backpropagation even across observations. To the best of our knowledge, no other model combines these three features together. We also provide additional theoretical results demonstrating universal approximation, and that Neural CDEs subsume alternative ODE models.
	
	\section*{Broader Impact}
	We have introduced a new tool for studying irregular time series. As with any tool, it may be used in both positive and negative ways. The authors have a particular interest in electronic health records (an important example of irregularly sampled time-stamped data) and so here at least we hope and expect to see a positive impact from this work. We do not expect any specific negative impacts from this work.

	\begin{ack}
	Thanks to Cristopher Salvi for many vigorous discussions on this topic. PK was supported by the EPSRC grant EP/L015811/1. JM was supported by the EPSRC grant EP/L015803/1 in collaboration with Iterex Therapuetics. JF was supported by the EPSRC grant EP/N509711/1. PK, JM, JF, TL were supported by the Alan Turing Institute under the EPSRC grant EP/N510129/1.
	\end{ack}
	
	\small
	\bibliographystyle{ieeetr}
	\bibliography{references} 
	
	\normalsize
	\newpage
	\appendix
	
	\begin{center}
	\huge Supplementary material
	\end{center}
	Appendix \ref{appendix:other-schemes-for-x} discusses the technical considerations of schemes for constructing a path $X$ from data. Appendix \ref{appendix:universal} proves universal approximation of the Neural CDE model, and is substantially more technical than the rest of this paper. Appendix \ref{appendix:nonlinear-g} proves that the Neural CDE model subsumes alternative ODE models which depend directly and nonlinearly on the data. Appendix \ref{appendix:experimental-details} gives the full details of every experiment, such as choice of optimiser, hyperparameter searches, and so on.
	
	\section{Other schemes for constructing the path $X$}\label{appendix:other-schemes-for-x}
	To evaluate the model as discussed in Section \ref{subsection:evaluate}, $X$ must be at least continuous and piecewise differentiable.
	
	\subsection{Differentiating with respect to the time points}
	However, there is a technical caveat in the specific case that derivatives with respect to the initial time $t_0$ are required, and that training is done with the adjoint method. In this case the derivative with respect to $t_0$ is computed using, and thus requires, derivatives of the vector field with respect to $t$.
	
	To be precise, suppose we have a Neural CDE as before:
	\begin{equation*}
	z_t = z_{t_0} + \int_{t_0}^t f_\theta(z_s) \mathrm{d} X_s\quad\text{for $t \in (t_0, t_n]$}.
	\end{equation*}
	Let $L$ be some (for simplicity scalar-valued) function of $z_{t_n}$, for example a loss. Consider
	\begin{equation*}
	g_{\theta, X}(z, s) = f_\theta(z) \frac{\mathrm{d}X}{\mathrm{d}s}(s)
	\end{equation*}
	as before, and let
	\begin{equation*}
	a_s = \frac{\mathrm{d}L}{\mathrm{d}z_s},
	\end{equation*}
	which is vector-valued, with size equal to the size of $z_s$, the number of hidden channels.
	
	Then, applying \cite[Equation 52]{neural-odes} to our case:
	\begin{align}
	\frac{\mathrm{d} L}{\mathrm{d} t_0} &= \frac{\mathrm{d} L}{\mathrm{d} t_n} -\int_{t_n}^{t_0} a_s \cdot \frac{\partial g_{\theta, X}}{\partial s}(z_s, s) \mathrm{d}s\nonumber\\
	&= \frac{\mathrm{d} L}{\mathrm{d} t_n} -\int_{t_n}^{t_0} a_s \cdot f_\theta(z_s) \frac{\mathrm{d}^2 X}{\mathrm{d} s^2}(s) \mathrm{d}s,\label{eq:t-derivative}
	\end{align}
	where $\cdot$ represents the dot product.
	
	In principle we may make sense of equation \eqref{eq:t-derivative} when $\ttfrac{\mathrm{d}^2 X}{\mathrm{d} s^2}$ is merely measure valued, but in practice most code is only set up to handle classical derivatives. If derivatives with respect to $t_0$ are desired, then practically speaking $X$ must be at least twice differentiable.
	
	\subsection{Adaptive step size solvers}
	There is one further caveat that must be considered. Suppose $X$ is twice differentiable, but that the second derivative is discontinuous. For example this would be accomplished by taking $X$ to be a quadratic spline interpolation.
	
	If seeking to solve equation \eqref{eq:t-derivative} with an adaptive step-size solver, we found that the solver would take a long time to compute the backward pass, as it would have to slow down to resolve each jump in $\ttfrac{\mathrm{d}^2 X}{\mathrm{d} s^2}$, and then speed back up in the intervals in-between.
	
	\subsection{Natural cubic splines}
	This is then the reason for our selection of natural cubic splines: by ensuring that $X$ is twice continuously differentiable, the above issue is ameliorated, and adaptive step size solvers operate acceptably. Cubic splines give essentially the minimum regularity for the techniques discussed in this paper to work `out of the box' in all cases.
	
	Other than this smoothness, however, there is little that is special about natural cubic splines. Other possible options are for example Gaussian processes \cite{gp-adapter1, gp-adapter2} or kernel methods \cite{interpolation-prediction}. Furthermore, especially in the case of noisy data it need not be an interpolation scheme -- approximation and curve-fitting schemes are valid too.
	
	\section{Universal Approximation}\label{appendix:universal}
	The behaviour of controlled differential equations are typically described through the \emph{signature transform} (also known as \emph{path signature} or simply \emph{signature}) \cite{levy-lyons} of the driving process $X$. 
	We demonstrate here how a (Neural) CDE may be reduced to consideration of just the signature transform, in order to prove universal approximation.
	
	The proof is essentially split into two parts. The first part is to prove universal approximation with respect to continuous paths, as is typically done for CDEs. The second (lengthier) part is to interpret what this means for the natural cubic splines that we use in this paper, so that we can get universal approximation with respect to the original data as well.
	
	\begin{definition}
	Let $\tau, T \in \reals$ with $\tau < T$ and let $v \in \naturals$. Let $\mathcal{V}^1([\tau, T]; \reals^v)$ represent the space of continuous functions of bounded variation. Equip this space with the norm
	\begin{equation*}
	\widehat{X} \mapsto \norm{\widehat{X}}_{\mathcal{V}} = \norm{\widehat{X}}_{\infty} + \abs{\widehat{X}}_{BV}.
	\end{equation*}
	\end{definition}
	
	This is a somewhat unusual norm to use, as bounded variation seminorms are more closely aligned with $L^1$ norms than $L^\infty$ norms.
	
	\begin{definition}
	For any $\widehat{X} \in \mathcal{V}^1([\tau, T]; \reals^v)$ let $X_t = (\widehat{X}_t, t) \in \mathcal{V}^1([\tau, T]; \reals^{v + 1})$. We choose to use the notation of `removing the hat' to denote this time augmentation, for consistency with the main text which uses $X$ for the time-augmented path.
	\end{definition}
	
	\begin{definition}
	For any $N, v \in \naturals$, let $\kappa(N, v) = \sum_{i = 0}^N (v + 1)^i$.
	\end{definition}
	
	\begin{definition}[Signature transform]\label{definition:signature}
	For any $k \in \naturals$ and any $y \in \reals^k$, let $M(y) \in \reals^{k(v + 1) \times (v + 1)}$ be the matrix
	\begin{equation*}
	M(y) = \begin{bmatrix}
	y^1 & 0 & 0 & \cdots & 0\\
	y^2 & 0 & 0 & \cdots & 0\\
	\vdots & \vdots & \vdots && \vdots\\
	y^k & 0 & 0 & \cdots & 0\\
	0 & y^1 & 0 & \cdots & 0\\
	0 & \vdots & 0 && \vdots\\
	0 & y^k & 0 & \cdots & 0 \\
	&&&\ddots\\
	0 & 0 & 0 & \cdots & y^1\\
	\vdots & \vdots & \vdots & \cdots & \vdots\\
	0 & 0 & 0 & \cdots & y^k\\
	\end{bmatrix}
	\end{equation*}
	
	Fix $N \in \naturals$ and $X \in \mathcal{V}^1([\tau, T]; \reals^{v + 1})$. Let $y^{0, X, N} \colon [\tau, T] \to \reals$ be constant with $y^{0, X, N}_t = 1$.
	
	For all $i \in \{1, \ldots, N\}$, iteratively let $y^{i, X, N} \colon [\tau, T] \to \reals^{(v + 1)^i}$ be the value of the integral
	\begin{equation*}
	y^{i, X, N}_t = y^{i, X, N}_\tau + \int_\tau^t M(y^{i - 1, X, N}_s) \mathrm{d}X_s\quad\text{for $t \in (\tau, T]$,}
	\end{equation*}
	with $y^{i, X, N}_\tau = 0 \in \reals^{(v + 1)^i}$.
	
	Then we may stack these together:
	\begin{align*}
	y^{X, N} = (y^{0, X, N}, \ldots, y^{N, X, N}) &\colon [\tau, T] \to \reals^{\kappa(N, v)},\\
	\widetilde{M}(y^{X, N}) = (0, M \circ y^{0, X, N}, \ldots, M \circ y^{N - 1, X, N}) &\colon [\tau, T] \to \reals^{\kappa(N, v) \times v}.
	\end{align*}
	Then $y^{X, N}$ is the unique solution to the CDE
	\begin{equation*}
	y^{X, N}_t = y^{X, N}_\tau + \int_\tau^t \widetilde{M}(y^{X, N})_s \mathrm{d}X_s\quad\text{for $t \in (\tau, T]$,}
	\end{equation*}
	with $y^{X, N}_\tau = (1, 0, \ldots, 0)$.
	
	Then the \emph{signature transform truncated to depth $N$} is defined as the map
	\begin{align*}
	\mathrm{Sig}^N &\colon \mathcal{V}^1([\tau, T]; \reals^{v + 1}) \to \reals^{\kappa(N, v)},\\
	\mathrm{Sig}^N &\colon X \mapsto y^{X, N}_T.
	\end{align*}
	
	If this seems like a strange definition, then note that for any $a \in \reals^k$ and any $b \in \reals^v$ that $M(a)b$ is equal to a flattened vector corresponding to the outer product $a \otimes b$. As such, the signature CDE is instead typically written much more concisely as the exponential differential equation
	\begin{equation*}
	y^{X, N}_t = y^{X, N}_\tau + \int_\tau^t y^{X, N}_s \otimes \mathrm{d}X_s\quad\text{for $t \in (\tau, T]$},
	\end{equation*}
	however we provide the above presentation for consistency with the rest of the text, which does not introduce $\otimes$.\newline
	\end{definition}
	
	\begin{definition}
	Let $\mathcal{V}^1_0([\tau, T]; \reals^v) = \set{X \in \mathcal{V}^1([\tau, T]; \reals^v)}{X_0 = 0}$.
	\end{definition}

	With these definitions out of the way, we are ready to state the famous universal nonlinearity property of the signature transform. We think \cite[Theorem 4.2]{perezarribas2018} gives the most straightforward proof of this result. This essentially states that the signature gives a basis for the space of functions on compact path space.
	\begin{thm}[Universal nonlinearity]\label{thm:universal-nonlinearity}
	Let $\tau, T \in \reals$ with $\tau < T$ and let $v, u \in \naturals$. 
	
	Let $K \subseteq \mathcal{V}^1_0([\tau, T]; \reals^v)$ be compact. (Note the subscript zero.)
	
	Let $\mathrm{Sig}^N \colon \mathcal{V}^1([\tau, T]; \reals^{v + 1}) \to \reals^{\kappa(N, v)}$ denote the signature transform truncated to depth $N$.
	
	Let
	\begin{equation*}
	J^{N, u} = \set{\ell \colon \reals^{\kappa(N, v)} \to \reals^u}{\ell \text{ is linear}}.
	\end{equation*}
	
	Then
	\begin{equation*}
	\bigcup_{N \in \naturals} \set{\widehat{X} \mapsto \ell(\mathrm{Sig}^N(X))}{\ell \in J^{N, u}}
	\end{equation*}
	is dense in $C(K; \reals^u)$.
	\end{thm}
	
	With the universal nonlinearity property, we can now prove universal approximation of CDEs with respect to controlling paths $X$.
	\begin{thm}[Universal approximation with CDEs]\label{prop:pathwise-universal-approximation}
	Let $\tau, T \in \reals$ with $\tau < T$ and let $v, u \in \naturals$. For any $w \in \naturals$ let
	\begin{align*}
	F^w &= \set{f \colon \reals^w \to \reals^{w \times (v + 1)}}{f \text{ is continuous}},\\
	L^{w, u} &= \set{\ell \colon \reals^w \to \reals^u}{\ell \text{ is linear}},\\
	\xi^w&= \set{\zeta \colon \reals^{v + 1} \to \reals^w}{\zeta \text{ is continuous}}.
	\end{align*}
	
	For any $w \in \naturals$, any $f \in F^w$ and any $\zeta \in \xi^w$ and any $\widehat{X} \in \mathcal{V}^1([\tau, T]; \reals^v)$, let $z^{f, \zeta, X} \colon [\tau, T] \to \reals^w$ be the unique solution to the CDE
	\begin{equation*}
	z^{f, \zeta, X}_t = z^{f, \zeta, X}_\tau + \int_\tau^t f(z^{f, \zeta, X}_s) \mathrm{d}X_s\quad\text{for $t \in (\tau, T]$,}
	\end{equation*}
	with $z^{f, \zeta, X}_\tau = \zeta(X_\tau)$.
	
	Let $K \subseteq \mathcal{V}^1([\tau, T]; \reals^v)$ be compact.
	
	Then
	\begin{equation*}
	\bigcup_{w \in \naturals} \set{\widehat{X} \mapsto \ell(z^{f, \zeta, X}_T)}{f \in F^w, \ell \in L^{w, u}, \zeta \in \xi^w}
	\end{equation*}
	is dense in $C(K; \reals^u)$.
	\end{thm}
	\begin{proof}
	We begin by prepending a straight line segment to every element of $K$. For every $\widehat{X} \in K$, define $\widehat{X}^* \colon [\tau - 1, T] \to \reals^v$ by
	\begin{equation*}
	\widehat{X}^*_t = \begin{cases}
	(t - \tau + 1)\widehat{X}_\tau & t \in [\tau -1 , \tau),\\
	\widehat{X}_t & t \in [\tau, T].
	\end{cases}
	\end{equation*}
	Similarly define $X^*$, so that a hat means that time is \emph{not} a channel, whilst a star means that an extra straight-line segment has been prepended to the path.
	
	Now let $K^* = \set{\widehat{X}^*}{\widehat{X} \in K}$. Then $K^* \subseteq \mathcal{V}^1_0([\tau - 1, T]; \reals^v)$ and is also compact. Therefore by Theorem \ref{thm:universal-nonlinearity},
	\begin{equation*}
	\bigcup_{N \in \naturals} \set{\widehat{X}^* \mapsto \ell(\mathrm{Sig}^N(X^*))}{\ell \in J^{N, u}}
	\end{equation*}
	is dense in $C(K^*; \reals^u)$.
	
	So let $\alpha \in C(K; \reals^u)$ and $\varepsilon > 0$. The map $\widehat{X} \mapsto \widehat{X}^*$ is a homeomorphism, so we may find $\beta \in C(K^*; \reals^u)$ such that $\beta(\widehat{X}^*) = \alpha(\widehat{X})$ for all $\widehat{X} \in K$. Next, there exists some $N \in \naturals$ and $\ell \in J^{N, u}$ such that $\gamma$ defined by $\gamma \colon \widehat{X}^* \mapsto \ell(\mathrm{Sig}^N(X^*))$ is $\varepsilon$-close to $\beta$.
	
	By Definition \ref{definition:signature} there exists $f \in F^{\kappa(N, v)}$ so that $\mathrm{Sig}^N(X^*) = y^{X^*}_T$ for all $X^* \in K^*$, where $y^{X^*}$ is the unique solution of the CDE
	\begin{equation*}
	y^{X^*}_t = y^{X^*}_{\tau - 1} + \int_{\tau - 1}^t f(y^{X^*}_s) \mathrm{d}X^*_s\quad\text{for $t \in (\tau - 1, T]$,}
	\end{equation*}
	with $y^{X^*}_{\tau - 1} = (1, 0, \ldots, 0)$.
	
	Now let $\zeta \in \xi$ be defined by $\zeta(X_\tau) = y^{X^*}_\tau$, which we note is well defined because the value of $y^{X^*}_t$ only depends on $X_\tau$ for $t \in [\tau - 1, \tau]$.
	
	Now for any $X \in K$ let $z^{X} \colon [\tau, T] \to \reals^w$ be the unique solution to the CDE
	\begin{equation*}
	z^{X}_t = z^{X}_\tau + \int_\tau^t f(z^{X}_s) \mathrm{d}X_s\quad\text{for $t \in (\tau, T]$,}
	\end{equation*}
	with $z^{X}_\tau = \zeta(X_\tau)$.
	
	Then by uniqueness of solution, $z^{X}_t = y^{X^*}_t$ for $t \in [\tau, T]$, and so in particular $\mathrm{Sig}^N(X^*) = y^{X^*}_T = z^{X}_T$.
	
	Finally it remains to note that $\ell \in J^{N, u} = L^{\kappa(N, v), u}$.
	
	So let $\delta$ be defined by $\delta \colon \widehat{X} \mapsto \ell(z^X_T)$. Then $\delta$ is in the set which we were aiming to show density of (with $w = \kappa(N, v)$, $f \in F^w$, $\ell \in L^{w, u}$ and $\zeta \in \xi$ as chosen above), whilst for all $\widehat{X} \in K$,
	\begin{equation*}
	\delta(\widehat{X}) = \ell(z^X_T) = \ell(\mathrm{Sig}^N(X^*)) = \gamma(\widehat{X}^*)
	\end{equation*}
	is $\varepsilon$-close to $\beta(\widehat{X}^*) = \alpha(\widehat{X})$. Thus density has been established.
	\end{proof}
	
	\begin{lemma}\label{lemma:compact}
	Let $K \subseteq C^2([\tau, T], \reals^v)$ be uniformly bounded with uniformly bounded first and second derivatives. That is, there exists some $C > 0$ such that $\norm{\widehat{X}}_\infty \!\!+ \norm{\ttfrac{\mathrm{d}\widehat{X}}{\mathrm{d}t}}_\infty \!+ \norm{\ttfrac{\mathrm{d}^2\widehat{X}}{\mathrm{d}t^2}}_\infty< C$ for all $\widehat{X} \in K$. Then $K \subseteq \mathcal{V}^1([\tau, T]; \reals^v)$ and is relatively compact (that is, its closure is compact) with respect to $\norm{\,\cdot\,}_\mathcal{V}$.
	\end{lemma}
	\begin{proof}
	$K$ is bounded in $C^2([\tau, T], \reals^v)$ so it is relatively compact in $C^1([\tau, T], \reals^v)$.
	
	Furthermore for any $\widehat{X} \in K$,
	\begin{align*}
	\norm{\widehat{X}}_\mathcal{V} &= \norm{\widehat{X}}_{\infty} + \abs{\widehat{X}}_{BV}\\
	&= \norm{\widehat{X}}_{\infty} + \norm{\frac{\mathrm{d}\widehat{X}}{\mathrm{d}t}}_1\\
	&\leq \norm{\widehat{X}}_{\infty} + \norm{\frac{\mathrm{d}\widehat{X}}{\mathrm{d}t}}_\infty,
	\end{align*}
	and so the embedding $C^1([\tau, T], \reals^v) \to \mathcal{V}^1([\tau, T]; \reals^v)$ is continuous. Therefore $K$ is also relatively compact in $\mathcal{V}^1([\tau, T]; \reals^v)$.
	\end{proof}
	
	Next we need to understand how a natural cubic spline is controlled by the size of its data. We establish the following crude bounds.
	
	\begin{lemma}\label{lemma:spline-bound}
	Let $v \in \naturals $. Let $x_0, \ldots, x_n \in \reals^v$. Let $t_0, \ldots, t_n \in \reals$ be such that $t_0 < t_1 < \cdots < t_n$. Let $\widehat{X} \colon [t_0, t_n] \to \reals^v$ be the natural cubic spline such that $\widehat{X}(t_i) = x_i$. Let $\tau_i = t_{i + 1} - t_i$ for all $i$. Then there exists an absolute constant $C > 0$ such that
	\begin{equation*}
	\norm{\widehat{X}}_\infty + \norm{\ttfrac{\mathrm{d}\widehat{X}}{\mathrm{d}t}}_\infty \!\!+ \norm{\ttfrac{\mathrm{d}^2\widehat{X}}{\mathrm{d}t^2}}_\infty < C \norm{\tau}_\infty \norm{x}_\infty (\min_i \tau_i)^{-2} (\norm{\tau}_\infty + (\min_i \tau_i)^{-1}).
	\end{equation*}
	\end{lemma}
	\begin{proof}
	Surprisingly, we could not find a reference for a fact of this type, but it follows essentially straightforwardly from the derivation of natural cubic splines.
	
	Without loss of generality assume $v = 1$, as we are using the infinity norm over the dimensions $v$, and each cubic interpolation is performed separately for each dimension.
	
	Let the $i$-th piece of $\widehat{X}$, which is a cubic on the interval $[t_i, t_{i + 1}]$, be denoted $Y_i$. Without loss of generality, translate each $Y_i$ to the origin so as to simplify the algebra, so that $Y_i \colon [0, \tau_i] \to \reals$. Let $Y_i(t) = a_i + b_i t + c_i t^2 + d_i t^3$ for some coefficients $a_i, b_i, c_i, d_i$ and $i \in \{0, \ldots, n - 1\}$.
	
	Letting $D_i = Y_i'(0)$ for $i \in \{0, \ldots, n - 1\}$ and $D_n = Y_{n - 1}(\tau_{n-1})$, the displacement and derivative conditions imposed at each knot are $Y_i(0) = x_i$, $Y_i(\tau_i) = x_{i + 1}$, $Y_i'(0) = D_i$ and $Y_i'(\tau_i) = D_{i + 1}$. This then implies that $a_i = x_i$, $b_i = D_i$,
	\begin{align}
	c_i &= 3 \tau_i^{-2} (x_{i + 1} - x_i) - \tau_i^{-1} (D_{i + 1} + 2 D_i),\label{eq:c}\\
	d_i &= 2 \tau_i^{-3} (x_i - x{i + 1}) + \tau_i^{-2} (D_{i + 1} + D_i).\label{eq:d}
	\end{align}
	
	Letting $\lesssim$ denote `less than or equal up to some absolute constant', then these equations imply that
	\begin{align}
	\norm{\widehat{X}}_\infty &= \max_i \norm{Y_i}_\infty \lesssim \max_i (\abs{x_i}  + \tau_i\abs{D_i}) \leq \norm{x}_\infty + \norm{\tau}_\infty \norm{D}_\infty,\label{eq:spline-bound-1}\\
	\norm{\frac{\mathrm{d}\widehat{X}}{\mathrm{d}t}}_\infty &= \max_i \norm{Y_i}_\infty \lesssim \max_i (\tau_i^{-1}\abs{x_i}  + \abs{D_i}) \leq \norm{x}_\infty (\min_i \tau_i)^{-1} + \norm{D}_\infty,\label{eq:spline-bound-2}\\
	\norm{\frac{\mathrm{d}^2 \widehat{X}}{\mathrm{d}t^2}}_\infty &= \max_i \norm{Y_i}_\infty \lesssim \max_i (\tau_i^{-2} \abs{x_i}  + \tau_i^{-1}\abs{D_i}) \leq \norm{x}_\infty (\min_i \tau_i)^{-2} + \norm{D}_\infty (\min_i \tau_i)^{-1}.\label{eq:spline-bound-3}
	\end{align}
	
	Next, the second derivative condition at each knot is $Y_{i - 1}''(\tau_{i - 1}) = Y_i''(0)$ for $i \in \{1, \ldots, n - 1\}$, and the natural condition is $Y_0''(0) = 0$ and $Y_{n - 1}''(\tau_{n - 1}) = 0$. With equations \eqref{eq:c}, \eqref{eq:d} this gives
	\begin{equation*}
	\mathcal{T} D = k,
	\end{equation*}
	where
	\begin{align*}
	\mathcal{T} &= 
	\begin{bmatrix}
	2 \tau_0^{-1} & \tau_0^{-1}\\
	\tau_0^{-1} & 2 (\tau_0^{-1} + \tau_1^{-1}) & \tau_1^{-1}\\
	 & \tau_1^{-1} & 2 (\tau_1^{-1} + \tau_2^{-1}) & \tau_2^{-1}\\
	 && \ddots & \ddots & \ddots\\
	 &&& \tau_{n - 2}^{-1} & 2(\tau_{n - 2}^{-1} + \tau_{n - 1}^{-1}) & \tau_{n - 1}^{-1}\\
	 &&&& \tau_{n - 1}^{-1} & 2\tau_{n - 1}^{-1}
	\end{bmatrix},\\
	D &= 
	\begin{bmatrix}
	D_0\\
	\vdots\\
	D_n
	\end{bmatrix},\\
	k &=
	\begin{bmatrix}
	3 \tau_0^{-2} (x_1 - x_0)\\
	3 \tau_1^{-2} (x_2 - x_1) + 3 \tau_0^{-2} (x_1 - x_0)\\
	\vdots\\
	3 \tau_{n - 1}^{-2} (x_n - x_{n - 1}) + 3 \tau_{n - 2}^{-2} (x_{n - 1} - x_{n - 2})\\
	3 \tau_{n - 1}^{-2} (x_n - x_{n - 1}).
	\end{bmatrix}
	\end{align*}
	
	Let $\norm{\mathcal{T}^{-1}}_\infty$ denote the operator norm and $\norm{D}_\infty$, $\norm{k}_\infty$ denote the elementwise norm. Now $\mathcal{T}$ is diagonally dominant, so the Varah bound \cite{varah} and HM-AM inequality gives
	\begin{equation*}
	\norm{\mathcal{T}^{-1}}_\infty \leq (\min_i (\tau_i^{-1} + \tau_{i + 1}^{-1}))^{-1} \lesssim \norm{\tau}_\infty.
	\end{equation*}
	
	Thus
	\begin{equation*}
	\norm{D}_\infty \lesssim \norm{\tau}_\infty \norm{k}_\infty \lesssim \norm{\tau}_\infty \norm{x}_\infty (\min_i \tau_i)^{-2}.
	\end{equation*}
	
	Together with equations \eqref{eq:spline-bound-1}--\eqref{eq:spline-bound-3} this gives the result.
	\end{proof}
	
	\begin{definition}[Space of time series]
	Let $v \in \naturals$. and $\tau, T \in \reals$ such that $\tau < T$. We define the space of time series in $[\tau, T]$ over $\reals^v$ as
	\begin{equation*}
	\ts{[\tau, T]}{\reals^v} = \set{((t_0, x_0), \ldots, (t_n, x_n))}{n \in \naturals, t_i \in [\tau, T], x_n \in \reals^v, t_0 = \tau, t_n = T, n \geq 2}.
	\end{equation*}
	\end{definition}
	
	To our knowledge, there is no standard topology on time series. One option is to treat them as sequences, however it is not clear how best to treat sequences of different lengths, or how to incorporate timestamp information. Given that a time series is typically some collection of observations from some underlying process, we believe the natural approach is to treat them as subspaces of functions.

	\begin{definition}[General topologies on time series]	
	Let $v \in \naturals$. and $\tau, T \in \reals$ such that $\tau < T$. Let $F$ denote some topological space of functions. Let $\iota \colon \ts{[\tau, T]}{\reals^v} \to F$ be some map. Then we may define a topology on $\ts{[\tau, T]}{\reals^v}$ as the weakest topology with respect to which $\iota$ is continuous.
	\end{definition}

	Recall that we use subscripts to denote function evaulation.
	
	\begin{definition}[Natural cubic spline topology]
	Let $v \in \naturals$. and $\tau, T \in \reals$ such that $\tau < T$. Let $F = C([\tau, T]; \reals^v)$ equipped with the uniform norm. For all $\mathbf{x} = ((t_0, x_0), \ldots, (t_n, x_n)) \in \ts{[\tau, T]}{\reals^v}$, let $\widehat{\iota} \colon \ts{[\tau, T]}{\reals^v} \to F$ produce the natural cubic spline such that $\widehat{\iota}(\mathbf{x})_{t_i} = x_i$ with knots at $t_0, \ldots, t_n$. Then this defines a topology on $\ts{[\tau, T]}{\reals^v}$ as in the previous definition.
	\end{definition}
	
	\begin{remark}
	In fact this defines a seminorm on $\ts{[\tau, T]}{\reals^v}$, by $\norm{\mathbf{x}} = \norm{\widehat{\iota}(\mathbf{x})}_\infty$. This is only a seminorm as for example $((0, 0), (2, 2))$ and $((0, 0), (1, 1), (2, 2))$ have the same natural cubic spline interpolation. This can be worked around so as to instead produce a full norm, but it is a deliberate choice not to: we would often prefer that these time series be thought of as equal. (And if it they are not equal, then first augmenting with observational intensity as in the main paper should distinguish them.)
	\end{remark}
	
	\begin{thm}[Universal approximation with Neural CDEs via natural cubic splines]\label{thm:datawise-universal-approximation}
	Let $\tau, T \in \reals$ with $\tau < T$ and let $v, u \in \naturals$. For any $w \in \naturals$ let
	\begin{align*}
	F^w_{\mathcal{N\!N}} &= \set{f \colon \reals^w \to \reals^{w \times (v + 1)}}{f \text{ is a feedforward neural network}},\\
	L^{w, u} &= \set{\ell \colon \reals^w \to \reals^u}{\ell \text{ is linear}},\\
	\xi^w_{\mathcal{N\!N}}&= \set{\zeta \colon \reals^{v + 1} \to \reals^w}{\zeta \text{ is a feedforward neural network}}.
	\end{align*}
	
	Let $\widehat{\iota}$ denote the natural cubic spline interpolation as in the previous definition, and recall that `removing the hat' is our notation for augmenting with time. For any $w \in \naturals$, any $f \in F^w$ and any $\zeta \in \xi^w_{\mathcal{N\!N}}$ and any $\mathbf{x} \in \ts{[\tau, T]}{\reals^v}$, let $z^{f, \zeta, \mathbf{x}} \colon [\tau, T] \to \reals^w$ be the unique solution to the CDE
	\begin{equation*}
	z^{f, \zeta, \mathbf{x}}_t = z^{f, \zeta, \mathbf{x}}_\tau + \int_\tau^t f(z^{f, \zeta, \mathbf{x}}_s) \mathrm{d}\iota(\mathbf{x})_s\quad\text{for $t \in (\tau, T]$,}
	\end{equation*}
	with $z^{f, \zeta, \mathbf{x}}_\tau = \zeta(\iota(\mathbf{x})_\tau)$.
	
	Let $K \subseteq \ts{[\tau, T]}{\reals^v}$ be such that there exists $C > 0$ such that
	\begin{equation}
	\norm{x}_\infty (\min_i (t_{i + 1} - t_i))^{-3} < C \label{eq:compact-require}
	\end{equation}
	for every $\mathbf{x} = ((t_0, x_0), \ldots, (t_n, x_n)) \in K$. (With $C$ independent of $\mathbf{x}$.)
	
	Then
	\begin{equation*}
	\bigcup_{w \in \naturals} \set{\mathbf{x} \mapsto \ell(z^{f, \zeta, \mathbf{x}}_T)}{f \in F^w_{\mathcal{N\!N}}, \ell \in L^{w, u}, \zeta \in \xi^w_{\mathcal{N\!N}}}
	\end{equation*}
	is dense in $C(K; \reals^u)$ with respect to the natural cubic spline topology on $\ts{[\tau, T]}{\reals^v}$.
	\end{thm}
	\begin{proof}
	Fix $\mathbf{x} = ((t_0, x_0), \ldots, (t_n, x_n)) \in K$. Let $\widehat{X} = \widehat{\iota}(\mathbf{x})$. Now $\norm{\tau}_\infty \leq T - \tau$ is bounded so by Lemma \ref{lemma:spline-bound} and the assumption of equation \eqref{eq:compact-require}, there exists a constant $C_1 > 0$ independent of $\mathbf{x}$ such that
	\begin{equation*}
	\norm{\widehat{X}}_\infty + \norm{\frac{\mathrm{d}\widehat{X}}{\mathrm{d}t}}_\infty + \norm{\frac{\mathrm{d}^2\widehat{X}}{\mathrm{d}t^2}}_\infty < C_1.
	\end{equation*}
	Thus by Lemma \ref{lemma:compact}, $\widehat{\iota}(K)$ is relatively compact in $\mathcal{V}^1([\tau, T]; \reals^v)$.
	
	Let $K_1 = \overline{\widehat{\iota}(K)}$, where the overline denotes a closure. Now by Theorem \ref{prop:pathwise-universal-approximation}, and defining $F^w$ and $\xi^w$ as in the statement of that theorem,
	\begin{equation*}
	\bigcup_{w \in \naturals} \set{\widehat{\iota}(\mathbf{x}) \mapsto \ell(z^{f, \zeta, \mathbf{x}}_T)}{f \in F^w, \ell \in L^{w, u}, \zeta \in \xi^w}
	\end{equation*}
	is dense in $C(K_1, \reals^u)$.
	
	For any $f \in F^w$, any $\zeta \in \xi^w$, any $f_{\mathcal{N\!N}} \in F^w_{\mathcal{N\!N}}$ and any $\zeta_{\mathcal{N\!N}} \in \xi^w_{\mathcal{N\!N}}$, the terminal values $z_T^{f, \zeta, \mathbf{x}}$ and $z_T^{f_{\mathcal{N\!N}}, \zeta_{\mathcal{N\!N}}, \mathbf{x}}$ may be compared by standard estimates, for example as commonly used in the proof of Picard's theorem. Classical universal approximation results for neural networks \cite{pinkus, deepandnarrow} then yield that
	\begin{equation*}
	\bigcup_{w \in \naturals} \set{\widehat{\iota}(\mathbf{x}) \mapsto \ell(z^{f, \zeta, \mathbf{x}}_T)}{f \in F^w_{\mathcal{N\!N}}, \ell \in L^{w, u}, \zeta \in \xi^w_{\mathcal{N\!N}}}
	\end{equation*}
	is dense in $C(K_1, \reals^u)$.
	
	By the definition of the natural cubic spline topology on $\ts{[\tau, T]}{\reals^v}$, then
	\begin{equation*}
	\bigcup_{w \in \naturals} \set{\mathbf{x} \mapsto \ell(z^{f, \zeta, \mathbf{x}}_T)}{f \in F^w_{\mathcal{N\!N}}, \ell \in L^{w, u}, \zeta \in \xi^w_{\mathcal{N\!N}}}
	\end{equation*}
	is dense in $C(K, \reals^u)$.
	\end{proof}
	
	\section{Comparison to alternative ODE models}\label{appendix:nonlinear-g}
	Suppose if instead of equation \eqref{eq:g-theta-x}, we replace $g_{\theta, X}(z, s)$ by $h_{\theta}(z, X_s)$ for some other vector field $h_\theta$. This might seem more natural. Instead of having $g_{\theta, X}$ be linear in $\ttfrac{\mathrm{d}X}{\mathrm{d}s}$, we take a $h_\theta$ that is potentially nonlinear in the control $X_s$.
	
	Have we gained anything by doing so? It turns out no, and in fact we have lost something. The Neural CDE setup directly subsumes anything depending directly on $X$.
	
	\begin{thm}\label{cont-rnn-thm}
	Let $\tau, T \in \reals$ with $\tau < T$, let $v, w \in \naturals$ with $v + 1 < w$. Let
	\begin{align*}
	F &= \set{f \colon \reals^w \to \reals^{w \times (v + 1)}}{f \text{ is continuous}},\\
	H &= \set{h \colon \reals^{w - v - 1} \times \reals^{v + 1} \to \reals^{w - v - 1}}{h \text{ is continuous}},\\
	\xi&= \set{\zeta \colon \reals^{v + 1} \to \reals^w}{\zeta \text{ is continuous}},\\
	\mathbb{X}&= \set{\widehat{X} \colon [\tau, T] \to \reals^v}{\widehat{X} \text{ continuous and of bounded variation}}.
	\end{align*}
	
	For any $\widehat{X} \in \mathbb{X}$, let $X_t = (\widehat{X}_t, t)$. Let $\pi \colon \reals^w \to \reals^{w - v - 1}$ be the orthogonal projection onto the first $w - v - 1$ coordinates.
	
	For any $f \in F$, any $\zeta \in \xi$, and any $\widehat{X} \in \mathbb{X}$, let $z^{f, \zeta, X} \colon [\tau, T] \to \reals^w$ be the unique solution to
	\begin{equation*}
	z^{f, \zeta, X}_t = z^{f, \zeta, X}_\tau + \int_\tau^t f(z^{f, \zeta, X}_s) \mathrm{d} X_s\quad\text{for $t \in (\tau, T]$,}
	\end{equation*}
	with $z^{f, \zeta, X}_\tau = \zeta(X_\tau)$.
	
	Similarly for any $h \in H$, any $\zeta \in \xi$, and any $\widehat{X} \in \mathbb{X}$, let $y^{f, X} \colon [\tau, T] \to \reals^{w - v - 1}$ be the unique solution to
	\begin{equation*}
	y^{h, \zeta, X}_t = y^{h, \zeta, X}_\tau + \int_\tau^t h(y^{h, \zeta, X}_s, X_s) \mathrm{d}s\quad\text{for $t \in (\tau, T]$,}
	\end{equation*}
	with $y^{h, \zeta, X}_\tau = \pi(\zeta(X_\tau))$.
	
	Let $\mathcal{Y} = \set{\widehat{X} \mapsto y^{h, \zeta, X}}{h \in H, \zeta \in \xi}$ and $\mathcal{Z} = \set{\widehat{X} \mapsto \pi \circ z^{f, \zeta, X}}{f \in F, \zeta \in \xi}$.
	
	Then $\mathcal{Y} \subsetneq \mathcal{Z}$.
	\end{thm}
	
	In the above statement, then a practical choice of $f \in F$ or $h \in H$ will typically correspond to some trained neural network.
	
	Note the inclusion of time via the augmentation $\widehat{X} \mapsto X$. Without it, then the reparameterisation invariance property of CDEs \cite{hambly2010uniqueness}, \cite[Proposition A.7]{kidger2019deep} will restrict the possible functions that CDEs can represent. This hypothesis is not necessary for the $\mathcal{Y} \neq \mathcal{Z}$ part of the conclusion.
	
	Note also how the CDE uses a larger state space of $w$, compared to $w - v - 1$ for the alternative ODE. The reason for this is that whilst $f$ has no explicit nonlinear dependence on $X$, we may construct it to have such a dependence implicitly, by recording $X$ into $v + 1$ of its $w$ hidden channels, whereupon $X$ is hidden state and may be treated nonlinearly. This hypothesis is also not necessary to demonstrate the $\mathcal{Y} \neq \mathcal{Z}$ part of the conclusion.
	
	This theorem is essentially an algebraic statement, and is thus not making any analytic claims, for example on universal approximation.
	
	\begin{proof}\quad\newline
	
	\boldheading{That $\mathcal{Y} \neq \mathcal{Z}$:}
	Let $z^{f, \zeta, \cdot} \in \mathcal{Z}$ for $\zeta \in \xi$ arbitrary and $f \in F$ constant and such that
	\begin{center}\begin{tikzpicture}
	 \matrix (vec) [matrix of math nodes, left delimiter = {[}, right delimiter = {]}] {
	1 & 0 & 0 & \cdots & 0\\
	0 & 0 & 0 & \cdots & 0\\
	\vdots & \vdots & \vdots & & \vdots\\
	0 & 0 & 0 & \cdots & 0\\
	};
	
	\node (a) at (vec-1-5.north) [right=15pt]{};
	\node (d) at (vec-4-5.south) [right=15pt]{};
	\draw [decorate, decoration={brace, amplitude=10pt}] (a) -- (d) node[midway, right=10pt] {\footnotesize $w$};
	
	\node (e) at (vec-4-1.center) [left=0pt]{};
	\draw [] (e) node[midway, left=40pt] {$f(z) = $};
	
	\node (f) at (vec-4-1.west) [below=10pt]{};
	\node (i) at (vec-4-5.east) [below=10pt]{};
	\draw [decorate, decoration={brace, amplitude=10pt}] (i) -- (f) node[midway, below=10pt] {\footnotesize $v + 1$};
	\end{tikzpicture}\end{center}
	
	Then for any $\widehat{X} \in \mathbb{X}$, the corresponding CDE solution in $\mathcal{Z}$ is
	\begin{equation*}
	z^{f, \zeta, X}_t = z^{f, \zeta, X}_{\tau} + \int_{\tau}^t f(z^{f, \zeta, X}_s) \mathrm{d}X_s,
	\end{equation*}
	and so the first component of its solution is
	\begin{equation*}
	z^{f, \zeta, X, 1}_t = X^1_t - X^1_\tau + \zeta^1(X_\tau),
	\end{equation*}
	whilst the other components are constant
	\begin{equation*}
	z^{f, \zeta, X, i}_t = \zeta^i(X_\tau)
	\end{equation*}
	for $i \in \{2, \ldots, w\}$, where superscripts refer to components throughout.
	
	Now suppose for contradiction that there exists $y^{h, \zeta, \cdot} \in \mathcal{Y}$ for some $\Xi \in \xi$ and $h \in H$ such that $y^{h, \Xi, X} = \pi \circ z^{f, \zeta, X}$ for all $\widehat{X} \in \mathbb{X}$. Now $y^{h, \Xi, X}$ must satisfy
	\begin{equation*}
	y^{h, \Xi, X}_t = y^{h, \Xi, X}_\tau + \int_\tau^t h(y^{h, \Xi, X}_s, X_s) \mathrm{d}s,
	\end{equation*}
	and so
	\begin{equation*}
	(X^1_t - X^1_{\tau} + \zeta^1(X_\tau), 0, \ldots, 0) = \pi(\Xi(X_\tau)) + \int_{\tau}^t h((X^1_s - X^1_{\tau} + \zeta^1(X_\tau), \zeta^2(X_\tau), \ldots, \zeta^w(X_\tau)), X_s) \mathrm{d}s.
	\end{equation*}
	Consider those $X$ which are differentiable. Differentiating with respect to $t$ now gives
	\begin{equation}\label{eq:impossible}
	\frac{\mathrm{d}X^1}{\mathrm{d}t}(t) = h^1((X^1_s - X^1_{\tau} + \zeta^1(X_\tau), \zeta^2(X_\tau), \ldots, \zeta^w(X_\tau)), X_t).
	\end{equation}
	
	That is, $h^1$ satisfies equation \eqref{eq:impossible} for all differentiable $X$. This is clearly impossible: the right hand side is a function of $t$,  $X_t$ and $X_{\tau}$ only, which is insufficient to determine $\ttfrac{\mathrm{d}X^1}{\mathrm{d}t}(t)$.
	
	\boldheading{That $\mathcal{Y} \subseteq \mathcal{Z}$:}
	Let $y^{h, \Xi, X} \in \mathcal{Y}$ for some $\Xi \in \xi$ and $h \in H$. Let $\sigma \colon \reals^w \to \reals^{v + 1}$ be the orthogonal projection onto the last $v + 1$ coordinates. Let $\zeta \in \xi$ be such that $\pi \circ \zeta = \pi \circ \Xi$ and $\sigma(\zeta(X_\tau)) = X_\tau$. Then let $f \in F$ be defined by
	\begin{center}\begin{tikzpicture}
	 \matrix (vec) [matrix of math nodes, left delimiter = {[}, right delimiter = {]}] {
	0 & 0 & \cdots & 0 & h^1(\pi(z), \sigma(z))\\
	\vdots & \vdots & & \vdots & \vdots\\
	0 & 0 & \cdots & 0 & h^{w - v - 1}(\pi(z), \sigma(z))\\
	1 & 0 & \cdots & 0 & 0\\
	0 & 1 & \cdots & 0 & 0\\
	\vdots & \vdots & \ddots & \vdots & \vdots\\
	0 & 0 & \cdots & 1 & 0\\
	0 & 0 & \cdots & 0 & 1\\
	};
	
	\node (a) at (vec-1-5.north) [right=60pt]{};
	\node (b) at (vec-3-5.south) [right=60pt]{};
	\node (c) at (vec-4-5.north) [right=60pt]{};
	\node (d) at (vec-8-5.south) [right=60pt]{};
	\draw [decorate, decoration={brace, amplitude=10pt}] (a) -- (b) node[midway, right=10pt] {\footnotesize $w - v - 1$};
	\draw [decorate, decoration={brace, amplitude=10pt}] (c) -- (d) node[midway, right=10pt] {\footnotesize $v + 1$};
	
	\node (e) at (vec-4-1.center) [left=0pt]{};
	\draw [] (e) node[midway, left=80pt] {$f(z) = $};
	
	\node (f) at (vec-8-1.west) [below=10pt]{};
	\node (g) at (vec-8-4.east) [below=10pt]{};
	\node (h) at (vec-8-4.east) [below=10pt]{};
	\node (i) at (vec-8-5.east) [below=10pt]{};
	\draw [decorate, decoration={brace, amplitude=10pt}] (g) -- (f) node[midway, below=10pt] {\footnotesize $v$};
	\draw [decorate, decoration={brace, amplitude=10pt}] (i) + (35pt,0) -- (h) node[midway, below=10pt] {\footnotesize $1$};	
	\end{tikzpicture}\end{center}
	
	Then for $t \in (\tau, T]$,
	\begin{align*}
	z^{f, \zeta, X}_t &= \zeta(X_\tau) + \int_\tau^t f(z^{f, \zeta, X}_s) \mathrm{d}X_s\\
	&=\zeta(X_\tau) + \int_\tau^t \begin{bmatrix}
	0 & 0 & \cdots & 0 & h^1(\pi(z^{f, \zeta, X}_s), \sigma(z^{f, \zeta, X}_s))\\
	\vdots & \vdots & & \vdots & \vdots\\
	0 & 0 & \cdots & 0 & h^{w - v - 1}(\pi(z^{f, \zeta, X}_s), \sigma(z^{f, \zeta, X}_s))\\
	1 & 0 & \cdots & 0 & 0\\
	0 & 1 & \cdots & 0 & 0\\
	\vdots & \vdots & \ddots & \vdots & \vdots\\
	0 & 0 & \cdots & 1 & 0\\
	0 & 0 & \cdots & 0 & 1\\
	\end{bmatrix}
	\begin{bmatrix}
	\mathrm{d}\widehat{X}^1_s\\
	\vdots\\
	\mathrm{d}\widehat{X}^v_s\\
	\mathrm{d}s
	\end{bmatrix}\\
	&= \zeta(X_\tau) + \int_\tau^t \begin{bmatrix}
	h^1(\pi(z^{f, \zeta, X}_s), \sigma(z^{f, \zeta, X}_s)) \mathrm{d}s\\
	\vdots\\
	h^{w - v - 1}(\pi(z^{f, \zeta, X}_s), \sigma(z^{f, \zeta, X}_s)) \mathrm{d}s\\
	\mathrm{d}\widehat{X}^1_s\\
	\vdots\\
	\mathrm{d}\widehat{X}^v_s\\
	\mathrm{d}s\\
	\end{bmatrix}\\
	&= \zeta(X_\tau) + \int_\tau^t \begin{bmatrix}
	h(\pi(z^{f, \zeta, X}_s), \sigma(z^{f, \zeta, X}_s)) \mathrm{d}s\\
	\mathrm{d}X_s
	\end{bmatrix}.
	\end{align*}
	
	Thus in particular
	\begin{equation*}
	\sigma(z^{f, \zeta, X}_t) = \sigma(\zeta(X_\tau)) + \int_\tau^t \mathrm{d}X_s = \sigma(\zeta(X_\tau)) - X_\tau + X_t = X_t.
	\end{equation*}
	Thus
	\begin{equation*}
	\pi(z^{f, \zeta, X}_t) = \pi(\zeta(X_\tau)) + \int_\tau^t h(\pi(z^{f, \zeta, X}_s), \sigma(z^{f, \zeta, X}_s)) \mathrm{d}s  = \pi(\Xi(X_\tau)) + \int_\tau^t h(\pi(z^{f, \zeta, X}_s), X_s) \mathrm{d}s.
	\end{equation*}
	
	Thus we see that $\pi(z^{f, \zeta, X})$ satisfies the same differential equation as $y^{h, \Xi, X}$. So by uniqueness of solution \cite[Theorem 1.3]{levy-lyons}, $y^{h, \Xi, X} = \pi(z^{f, \zeta, X}) \in \mathcal{Z}$.
	\end{proof}
	
	\section{Experimental details}\label{appendix:experimental-details}
	
	\subsection{General notes}
	\boldheading{Code}
	Code to reproduce every experiment can be found at \url{https://github.com/patrick-kidger/NeuralCDE}.
	
	\boldheading{Normalisation}
	Every dataset was normalised so that each channel has mean zero and variance one.
	
	\boldheading{Loss}
	Every binary classification problem used binary cross-entropy loss applied to the sigmoid of the output of the model. Every multiclass classification problem used cross-entropy loss applied to the softmax of the output of the model.

	\boldheading{Architectures}
	For both the Neural CDE and ODE-RNN, the integrand $f_\theta$ was taken to be a feedforward neural network. A final linear layer was always used to map from the terminal hidden state to the output.
	
	\boldheading{Activation functions}
	For the Neural CDE model we used ReLU activation functions. Following the recommendations of \cite{latent-odes}, we used tanh activation functions for the ODE-RNN model, who remark that for the ODE-RNN model, tanh activations seem to make the model easier for the ODE solver to resolve. Interestingly we did not observe this behaviour when trying tanh activations and \texttt{method=`dopri5'} with the Neural CDE model, hence our choice of ReLU.
	
	\boldheading{Optimiser}
	Every problem used the Adam \cite{kingma2015} optimiser as implemented by PyTorch 1.3.1 \cite{pytorch}. Learning rate and batch size varied between experiments, see below. The learning rate was reduced if metrics failed to improve for a certain number of epochs, and training was terminated if metrics failed to improve for a certain (larger) number of epochs. The details of this varied by experiment, see the individual sections. Once training was finished, then the parameters were rolled back to the parameters which produced the best validation accuracy throughout the whole training procedure. The learning rate for the final linear layer (mapping from the hidden state of a model to the output) was typically taken to be much larger than the learning rate used elsewhere in the model; this is a standard trick that we found improved performance for all models.
	
	\boldheading{Hyperparameter selection}
	In brief, hyperparameters were selected to optimise the ODE-RNN baseline, and equivalent hyperparameters used for the other models.
	
	In more detail:
	
	We began by selecting the learning rate. This was selected by starting at 0.001 and reducing it until good performance was achieved for a small ODE-RNN model with batch size 32.
	
	After this, we increased the batch size until the selected model trained at what was in our judgement a reasonable speed. As is standard practice, we increased the learning rate proportionate to the increase in batch size.
	
	Subsequently we selected model hyperparameters (number of hidden channels, width and depth of the vector field network) via a grid search to optimise the ODE-RNN baseline. A single run of each hyperparameter choice was performed. The equivalent hyperparameters were then used on the GRU-$\Delta$t, GRU-D, GRU-ODE baselines, and also our Neural CDE models, after being adjusted to produce roughly the same number of parameters for each model.
	
	The grids searched over and the resulting hyperparameters are stated in the individual sections below.
	
	\boldheading{Weight regularisation}
	$L^2$ weight regularisation was applied to every parameter of the ODE-RNN, GRU-$\Delta$t and GRU-D models, and to every parameter of the vector fields for the Neural CDE and GRU-ODE models.
	
	\boldheading{ODE Solvers}
	The ODE components of the ODE-RNN, GRU-ODE, and Neural CDE models were all computed using the fourth-order Runge-Kutta with 3/8 rule solver, as implemented by passing \texttt{method=`rk4'} to the \texttt{odeint\_adjoint} function of the \texttt{torchdiffeq} \cite{torchdiffeq} package. The step size was taken to equal the minimum time difference between any two adjacent observations.
		
	\boldheading{Adjoint backpropagation}
	The GRU-ODE, Neural CDE and the ODE component of the ODE-RNN are all trained via the adjoint backpropagation method \cite{neural-odes}, as implemented by \texttt{odeint\_adjoint} function of the \texttt{torchdiffeq} package.
	
	\boldheading{Computing infrastructure}	
	All experiments were run on one of two computers; both used Ubuntu 18.04.4 LTS, were running PyTorch 1.3.1, and used version 0.0.1 of the \texttt{torchdiffeq} \cite{torchdiffeq} package. One computer was equipped with a Xeon E5-2960 v4, two GeForce RTX 2080 Ti, and two Quadro GP100, whilst the other was equipped with a Xeon Silver 4104 and three GeForce RTX 2080 Ti.

	\subsection{CharacterTrajectories}\label{appendix:uea}
	The learning rate used was 0.001 and the batch size used was 32. If the validation loss stagnated for 10 epochs then the learning rate was divided by 10 and training resumed. If the training loss or training accuracy stagnated for 50 epochs then training was terminated. The maximum number of epochs allowed was 1000.
	
	We combined the train/test split of the original dataset (which are of unusual proportion, being 50\%/50\%), and then took a 70\%/15\%/15\% train/validation/test split.
	
	The initial condition $\zeta_\theta$ of the Neural CDE model was taken to be a learnt linear map from the first observation to the hidden state vector. (Recall that is an important part of the model, to avoid translation invariance.)
	
	The hyperparameters were optimised (for just the ODE-RNN baseline as previously described) by performing most of a grid search over 16 or 32 hidden channels, 32, 48, 64, 128 hidden layer size, and 1, 2, 3 hidden layers. (The latter two hyperparameters corresponding to the vector fields of the ODE-RNN and Neural CDE models.) A few option combinations were not tested due to the poor performance of similar combinations. (For example every combination with hidden layer size of 128 demonstrated relatively poor performance.) The search was done on just the 30\% missing data case, and the same hyperparameters were used for the 50\% and 70\% missing data cases.
	
	The hyperparameters selected were 32 hidden channels for the Neural CDE and ODE-RNN models, and 47 hidden channels for the GRU-$\Delta$t, GRU-D and GRU-ODE models. The Neural CDE and ODE-RNN models both used a feedforward network for their vector fields, with 3 hidden layers each of width 32. The resulting parameter counts for each model were 8212 for the Neural CDE, 8436 for the ODE-RNN, 8386 for the GRU-D, 8292 for the GRU-$\Delta$t, and 8372 for the GRU-ODE.
	
	\subsection{PhysioNet sepsis prediction}\label{appendix:sepsis}
	The batch size used was 1024 and learning rate used was 0.0032, arrived at as previously described. If the training loss stagnated for 10 epochs then the learning rate was divided by 10 and training resumed. If the training loss or validation accuracy stagnated for 100 epochs then training was terminated. The maximum number of epochs allowed was 200. The learning rate for the final linear layer (a component of every model, mapping from the final hidden state to the output) used a learning rate that 100 times larger, so 0.32.
	
	The original dataset does not come with an existing split, so we took our own 70\%/15\%/15\% train/validation/test split.
	
	As this problem featured static (not time-varying) features, we incorporated this information by allowing the initial condition of every model to depend on these. This was taken to be a single hidden layer feedforward network with ReLU activation functions and of width 256, which we did not attempt a hyperparameter search over.
	
	As this dataset is partially observed, then for the ODE-RNN, GRU-$\Delta$t, GRU-D models, which require \emph{something} to be passed at each time step, even if a value is missing, then we fill in missing values with natural cubic splines, for ease of comparison with the Neural CDE and ODE-RNN models. (We do not describe this as imputation as for the observational intensity case the observational mask is additionally passed to these models.) In particular this differs slightly from the usual implementation of GRU-D, which usually use a weighted average of the last observation and the mean. Splines accomplishes much the same thing, and help keep things consistent between the various models.
	
	The hyperparameters were optimised (for just the ODE-RNN baseline as previously described) by performing most of a grid search over 64, 128, 256 hidden channels, 64, 128, 256 hidden layer size, and 1, 2, 3, 4 hidden layers. (The latter two hyperparameters corresponding to the vector fields of the ODE-RNN and Neural CDE models.)
	
	The hyperparameters selected for the ODE-RNN model were 128 hidden channels, and a vector field given by a feedforward neural network with hidden layer size 128 and 4 hidden layers. In order to keep the number of parameters the same between each model, this was reduced to 49 hidden channels and hidden layer size 49 for the Neural CDE model, and increased to 187 hidden channels for the GRU-$\Delta$t, GRU-D and GRU-ODE models. When using observational intensity, the resulting parameter counts were 193541 for the Neural CDE, 194049 for the ODE-RNN, 195407 for the GRU-D, 195033 for the GRU-$\Delta$t, and 194541 for the GRU-ODE. When not using observational intensity, the resulting parameter counts were 109729 for the Neural CDE, 180097 for the ODE-RNN, 175260 for the GRU-D, 174886 for the GRU-$\Delta$t, and 174921 for the GRU-ODE. Note the dramatically reduced parameter count for the Neural CDE; this is because removing observational intensity reduces the number of channels, which affects the parameter count dramatically as discussed in Section \ref{section:limitations}.
	
	\subsection{Speech Commands}\label{appendix:speechcommands}
	The batch size used was 1024 and the learning rate used was 0.0016, arrived at as previously described. If the training loss stagnated for 10 epochs then the learning rate was divided by 10 and training resumed. If the training loss or validation accuracy stagnated for 100 epochs then training was terminated. The maximum number of epochs allowed was 200. The learning rate for the final linear layer (a component of every model, mapping from the final hidden state to the output) used a learning rate that 100 times larger, so 0.16.
	
	Each time series from the dataset is univariate and of length 16000. We computed 20 Mel-frequency cepstral coefficients of the input as implemented by \texttt{torchaudio.transforms.MFCC}, with logarithmic scaling applied to the coefficients. The window for the short-time Fourier transform component was a Hann window of length 200, with hop length of 100, with 200 frequency bins. This was passed through 128 mel filterbanks and 20 mel coefficients extracted. This produced a time series of length 161 with 20 channels. We took a 70\%/15\%/15\% train/validation/test split.
	
	The hyperparameters were optimised (for just the ODE-RNN baseline as previously described) by performing most of a grid search over 32, 64, 128 hidden channels, 32, 64, 128 hidden layer size, and 1, 2, 3, 4 hidden layers. (The latter two hyperparameters corresponding to the vector fields of the ODE-RNN and Neural CDE models.)
	
	The hyperparameters selected for the ODE-RNN model were 128 hidden channels, and a vector field given by a feedforward neural network with hidden layer size 64 and 4 hidden layers. In order to keep the number of parameters the same between each model, this was reduced to 90 hidden channels and hidden layer size 40 for the Neural CDE model, and increased to 160 hidden channels for the GRU-$\Delta$t, GRU-D and GRU-ODE models. The resulting parameter counts were 88940 for the Neural CDE model, 87946 for the ODE-RNN model, 89290 for the GRU-D model, 88970 for the GRU-dt model, and 89180 for the GRU-ODE model.
\end{document}